\documentclass{article}
\usepackage{kr}

\usepackage{times}
\usepackage{soul}
\usepackage{url}
\usepackage[hidelinks]{hyperref}
\usepackage[utf8]{inputenc}
\usepackage[small]{caption}
\usepackage{graphicx}
\usepackage{amsmath}
\usepackage{amsthm}
\usepackage{amssymb}
\usepackage{booktabs}
\usepackage{algorithm}
\usepackage[noend]{algorithmic}
\usepackage{xspace}
\usepackage{listings}
\usepackage{xcolor}
\usepackage{threeparttable}
\usepackage{tikz}
\usetikzlibrary{arrows.meta,positioning}
\usepackage{placeins}

\def\preds{\mathbf{R}}
\def\gpreds{\preds_{\mathsf{A}}}
\def\opreds{\preds_{\mathsf{O}}}
\def\epreds{\preds_{\mathsf{E}}}
\def\pspreds{\preds_{\mathsf{PS}}}
\def\nspreds{\preds_{\mathsf{NS}}}
\def\metapreds{\preds_{\mathsf{M}}}
\def\const{\mathbf{C}}

\def\vars{\mathbf{V}}
\def\dvars{\vars_{\mathsf{d}}}
\def\tvars{\vars_{\mathsf{t}}}

\def\now{\ast}
\def\simplerules{\Pi_{\mathsf{SE}}}
\def\metarules{\Pi_{\mathsf{ME}}}
\def\constraints{\Upsilon}
\def\tconstrain{\constraints_{\mathsf{temp}}}
\def\domconstrain{\constraints_{\mathsf{dom}}}
\def\eventspec{\Sigma}
\def\dataset{\mathcal{D}}
\def\existpred{\mathsf{exists}}
\def\termpred{\mathsf{ends}}
\def\windowpred{\mathsf{window}}

\def\evalcond{\textsc{Eval}}

\def\stop{\times}

\def\intersect{\mathsf{inter}}

\newcommand{\vect}[1]{\boldsymbol{#1}}

\def\infevents{\mathsf{SE}(\dataset, \eventspec)}
\def\infeventsnoconf{\mathsf{SE}^-(\dataset, \eventspec)}
\def\infermeta{\mathsf{ME}(\dataset, \mathcal{S}, \eventspec)}
\def\infermetanoconf{\mathsf{ME}^-(\dataset, \mathcal{S}, \eventspec)}

\def\reps{\mathsf{Reps}}
\def\prefreps{\mathsf{PrefReps}}

\newcommand{\casper}{\textbf{\textsc{heva}}\xspace}
\newcommand{\ncasper}{\textbf{\textsc{heva}}}

\newtheorem{example}{Example}
\newtheorem{theorem}{Theorem}
\newtheorem{lemma}{Lemma}
\newtheorem{definition}{Definition}
\newtheorem{proposition}{Proposition}

\definecolor{codegray}{gray}{0.95}
\definecolor{commentblue}{rgb}{0.0, 0.0, 0.6}
\definecolor{predicategreen}{rgb}{0.0, 0.5, 0.0}

\lstdefinelanguage{ASP}{
  alsoletter={:,-,'},           
  morecomment=[l]{\%},            
  morecomment=[s]{/*}{*/},        
  morestring=[b]",                
  sensitive=true,
  morekeywords={
    not                          
  },
literate=
    {:-}{{\textcolor{teal!70!black}{:-}}}2
    {..}{{\textcolor{teal!70!black}{..}}}2
    {|}{{\textcolor{teal!70!black}{|}}}1
    {\{}{{\textcolor{teal!70!black}{\{}}}1
    {\}}{{\textcolor{teal!70!black}{\}}}}1
}

\lstdefinestyle{aspstyle}{
  language=ASP,
  backgroundcolor=\color{cyan!10},
  basicstyle=\ttfamily\footnotesize,
  frame=single,
  framerule=0pt,
  keywordstyle=\bfseries\color{blue!60!black},
  commentstyle=\itshape\color{gray!70!black},
  stringstyle=\color{green!40!black},
  numberstyle=\tiny\color{gray},
  numbers=left, numbersep=8pt,
  showstringspaces=false,
  keepspaces=true,
  columns=fixed,
  tabsize=2,
  framesep=3pt,
  breaklines=true,
  captionpos=b,
  keywords={exists, exists_pers, terminates, obs, pt_window, m_event, intersection_of, start, end, pre_candidate, l_candidate, event, persist_end, rep_event, inconsistent_if_added, contains, temporal_conflict_with_lower_level}
}

\title{Inferring High-Level Events from Timestamped Data:\\ Complexity and Medical Applications}

\author{%
Yvon K. Awuklu$^{1,2,3}$\and
Meghyn Bienvenu$^1$\and
Katsumi Inoue$^4$\and
Vianney Jouhet$^{2,3}$\and 
Fleur Mougin$^3$ \\
\affiliations
$^1$Univ. Bordeaux, CNRS, Bordeaux INP, LaBRI, UMR 5800, F-33400, Talence, France\\
$^2$CHU de Bordeaux, Service d’Information Médicale, F-33000, Bordeaux, France\\
$^3$Univ. Bordeaux, INSERM, BPH, U1219, F-33000, Bordeaux, France\\
$^4$National Institute of Informatics, Tokyo, Japan \\
\emails
\{{kokou-yvon.awuklu,fleur.mougin,meghyn.bienvenu\}@u-bordeaux.fr,\\vianney.jouhet@chu-bordeaux.fr,inoue@nii.ac.jp}
}

\begin{document}

\maketitle

\begin{abstract}
In this paper, we develop a novel logic-based approach to detecting high-level temporally extended events from time\-stamped data and background knowledge. Our framework employs logical rules to capture existence and termination conditions for simple temporal events and to combine these into meta-events. In the medical domain, for example, di\-sease episodes and therapies are inferred from timestamped clinical observations, such as diagnoses and drug administrations stored in patient records, and can be further combined into higher-level disease events. As some incorrect events might be inferred, we use constraints to identify incompatible combinations of events and propose a repair mechanism to select preferred consistent sets of events. While reasoning in the full framework is intractable, we identify relevant restrictions that ensure polynomial-time data complexity. Our prototype system implements core components of the approach using answer set programming. An evaluation on a lung cancer use case supports the interest of the approach, both in terms of computational feasibility and positive alignment of our results with medical expert opinions. While strongly motivated by the needs of the healthcare domain, our framework is purposely generic, enabling its reuse in other 
areas.
\end{abstract}

\section{Introduction}

The adoption of electronic health records (EHRs) has significantly improved access to vast amounts of clinical data. While this access is invaluable for healthcare delivery and research, it also introduces new challenges due to the inherent complexity of medical data~\cite{tsai_Effects_2020,rance_Integrating_2016}, and in particular, its inherently temporal nature~\cite{Augusto2005-ke,Zhou2007-fw,li_Time_2020,awuklu_Ontologydriven_2025a}. 
In practice, data in EHRs is recorded as a sequence of time-stamped observations, such as test results, diagnoses, and treatments. Each observation implicitly refers to a clinical event—an underlying phenomenon such as a disease episode or a therapeutic intervention—that is relevant to the patient's care~\cite{hripcsak_Nextgeneration_2013}, but these events are not explicitly documented in EHRs. 
Physicians routinely infer clinical events using their medical expertise, which is rarely formalized in the data. For example, observing repeated antibiotic intake may lead a clinician to deduce the presence of a bacterial infection. Here, the data shows the observation (antibiotic intake), while the associated clinical events (bacterial infection, antibiotic therapy) and the reasoning process remain implicit. 
However, as the volume of observations 
per patient can overwhelm clinicians, there is a need for information systems 
to adopt 
a higher-level, event-based representation more closely aligned 
with clinical reasoning and decision-making. 
Moreover, given the high-stakes nature of healthcare decisions and medical research, 
the event inference process should also be transparent, making it possible to trace  inferred high-level events back to the original observations. 
These 
considerations motivate us to develop
a new logic-based framework for identifying high-level temporally extended events,
designed with the medical domain in mind yet sufficiently generic to enable reuse in other 
areas.
\medskip

\noindent\textbf{Overview of the Logical Framework}
We briefly outline the main intuitions and components of our framework, 
which borrows ideas from various existing KR formalisms (see Section \ref{related} for a comparison with related work). 
Recall that our aim is to be able to infer temporally-extended events like $
\mathsf{ABTherapy}
(p,d,[t_1,t_2])$,
expressing that patient $p$ receives antibiotic therapy with drug $d$ during the time period $[t_1,t_2]$, 
from the timestamped observations in medical data. 
The starting point for our proposal is the realisation that 
while 
it can be quite difficult for medical experts to directly specify the interval endpoints ($t_1, t_2$), 
it is typically much easier for them to supply \emph{existence conditions} which ensure (or suggest) 
that a given event is ongoing at timepoint $t$, e.g.\ the
observation $\mathsf{DrugAdmin}(p,d,t)$ available in patient $p$'s records 
indicates that a therapy with drug $d$ is ongoing at time $t$. In some cases, experts may also be able to 
formulate \emph{termination conditions} that indicate that an event (might or must) end at a given timepoint. 
For example, hospital records may contain 
$\mathsf{DrugStopNotification}(p,d,t)$, 
stating that the prescription for drug $d$ was terminated at time $t$,
indicating an end to the therapy with $d$.  
In our work, we will use the term \emph{simple event} to refer to events for which we can define existence conditions 
(and optionally termination conditions) in terms of data predicates, without reference to other events. 

How can we use the existence and termination conditions to identify the intervals of simple events?
Here, we shall distinguish two kinds of simple events: \emph{persistent} and \emph{non-persistent}. 
Intuitively, persistent events continue once initiated until some 
termination condition is met (similar to the principle of inertia in reasoning about actions \cite{reiter-book,DBLP:books/daglib/0095085}), 
while non-persistent events require regular observations to endure. 
For persistent events, we consider the maximal intervals starting from a timepoint verifying the existence 
condition and continuing until some termination condition (if present) is reached (else marking the event as ongoing). 
For non-persistent events, the idea is to group together timepoints satisfying existence conditions
if they are sufficiently close (with `closeness' being determined by a provided time window), ending 
an event when a termination condition is reached or there are no further timepoints in the vicinity 
that satisfy existence conditions. 

Our framework also supports \emph{meta-events}, defined from simple events and possibly other meta-events. 
Such higher-level events allow one e.g.\ to group disease episodes or identify concurrent conditions or treatments. 
They also make it possible to present events at different levels of abstraction 
(e.g.\ $\mathsf{DrugTherapy}$ generalizes $\mathsf{ABTherapy}$). 
Moreover, since it is difficult to provide existence and 
termination conditions which are both fully accurate and ensure sufficient coverage, 
our framework allows for event facts to be annotated with confidence levels, 
based upon the rules which created them. Constraints can also be used to specify
consistency requirements, and a \emph{repair mechanism} employed to select consistent combinations of events.\medskip

\noindent\textbf{Contributions}
Our first contribution is the formal definition of a novel logic-based language for temporal event detection,
which uses rules to specify the existence and termination conditions for simple 
events and for defining meta-events in terms of other events, as well as constraints to define consistency requirements. 
The semantics defines four different kinds of timelines (naïve, consistent, preferred, cautious), with each timeline
consisting of simple event and meta-event facts that can be inferred from the rules, possibly taking into account the constraints and confidence levels. 

As our second contribution, we explore the computational properties of our framework, presenting 
algorithms and complexity results for recognizing and generating 
the different kinds of timeline. 
Given the expressivity of the framework, we show unsurprisingly that 
the consistent, preferred, and cautious cannot be tractably recognized.
However, we also identify 
a relevant fragment for which there is a unique preferred timeline,
and for which both the preferred timeline and cautious timeline can be efficiently computed. 

As a third contribution, we 
implemented core components of the framework using answer set programming (ASP), a well-known declarative 
programming paradigm \cite{DBLP:journals/cacm/BrewkaET11,asp-in-prac-2012,DBLP:books/sp/Lifschitz19}. 
An experimental evaluation 
on a cancer use case involving hospital data 
demonstrates computational feasibility,
and feedback from medical experts supports the clinical plausibility of the inferred events. \medskip

\noindent\textbf{Paper Organization} Section \ref{secframework} defines the syntax and semantics of 
our logical framework, while Section \ref{complexity} provides algorithms and complexity results for the relevant reasoning tasks. 
Sections \ref{casper} and \ref{eval} present respectively the implemented system and its evaluation on a medical use case. 
Discussions of related approaches and future work are given in 
Sections \ref{related}  and \ref{conclusion}.
Omitted proofs and further details on the modelling and evaluation of the case study  
are given in the appendix. 

\section{A Logical Framework for Event Detection}\label{secframework}

We formalize our logical framework for 
specifying high-level events from temporal observations and expert know\-ledge, motivated and illustrated by medical applications. 
We shall assume that readers are acquainted with the basics of first-order logic and logic programming. 


\def\tstarvars{\tvars^{\now}}
\def\natnum{\mathbb{N}}
\def\numfn{\mathbf{F}_\mathsf{n}}
\def\numvars{\vars_\mathsf{n}}
\def\intvars{\vars_\mathsf{int}}
\def\intfn{\mathbf{F}_\mathsf{int}}
\def\terms{\mathbf{T}}
\def\numterms{\mathbf{T}_\mathsf{n}}
\def\dterms{\mathbf{T}_\mathsf{d}}
\def\intterms{\mathbf{T}_\mathsf{int}}

\subsection{Logical Vocabulary}
We use three disjoint sets of first-order predicates to describe the domain:
a set $\gpreds$ of atemporal predicates, a set of $\opreds$ of observation predicates, 
and a set $\epreds$ of event pre\-dicates, 
further partitioned into the sets $\pspreds$, $\nspreds$, and $\metapreds$
of persistent simple event predicates,
non-persistent simple event predicates, and meta-event predicates. 
Each predicate $R \in \gpreds \cup \opreds \cup \epreds$
has an arity $k \geq 0$, corresponding to its number of atemporal arguments. 
We use $\gpreds^k$ (resp.\ $\opreds^k$, $\epreds^k$) for the set of $k$-ary 
predicates in $\gpreds$ (resp.\ $\opreds, \epreds$). 
The atemporal and observation predicates correspond to the predicates 
occurring in the data,  the key difference being that observation predicates 
have a timepoint as final argument (to capture timestamped facts). 
The event predicates are used for the inferred events and will contain 
a temporal interval as argument to indicate the start and end of the event. 

For simplicity, we use natural numbers to represent timepoints (and positive integers for
confidence levels and temporal windows), so we assume the set 
 $\const$ of \emph{constants} contains $\natnum$. \emph{Temporal intervals} take the 
form $[t_1,t_2]$ where $t_1\in \mathbb{N}$, $t_2 \in \mathbb{N} \cup \{\now\}$, 
and $t_1 \leq t_2$ if $t_2 \neq \now$, where the special symbol $\now$ is 
used to indicate events which are still ongoing. 
Our rules will use \emph{variables} drawn from a set $\vars$. 
To ensure variables are appropriately instantiated, 
we assume $\vars$ is partitioned into  four subsets $\dvars$, $\numvars$, $\numvars^+$, and $\numvars^{\now}$ whose variables must be instantiated respectively by ele\-ments of $\const$, $\natnum$, $\natnum^+$, and $\natnum \cup \{\now\}$. 
We also allow for sets $\numfn$, $\numfn^+$, and $\intfn$ of functions to manipulate (positive) natu\-ral numbers or intervals. For example, we will use $\mathsf{min}$ to aggregate confidence levels 
and $\mathsf{inter}$ to compute the intersection of two intervals. 
Finally, we can define the following sets of \emph{terms}: (i) $\dterms= \const \cup \dvars$,  (ii)
$\numterms$ (resp.\ $\numterms^+$) is defined as the closure of $\natnum \cup \numvars$ (resp.\ 
$\natnum^+ \cup \numvars^+$) under applications of functions in $\numfn$ (resp.\ $\numfn^+$), 
and (iii) $\intterms$ is obtained by closing $\{[t_1, t_2] \mid t_1 \in \natnum \cup \numvars, t_2 \in \natnum \cup \{\now\} \cup \numvars^{\now}, t_1 \leq t_2 \text{ if } t_1,t_2 \in \natnum\}$  
under functions in $\intfn$. 

An \emph{atemporal atom} has the form $R(u_1, \ldots, u_k)$ where $R \in \gpreds^k$, and $u_1, \ldots, u_k \in \dterms$.
An \emph{observation atom} has the form 
$R(u_1, \ldots, u_k, t)$ where $R \in \opreds^k$, $u_1, \ldots, u_k \in \dterms$
and $t \in \numterms$.
An \emph{event atom} has the form 
$R(u_1, \ldots, u_k, \iota)$, where $R \in \epreds^k$,  $u_1, \ldots, u_k \in \dterms$,
$\iota \in \intterms$. 
Atoms without variables are called \emph{facts}.  
A \emph{dataset} is a finite set of atemporal and observation facts. 

We will attach \emph{confidence levels} (from $\natnum^+$) to event facts based upon how they were gene\-rated. 
Note that to easily identify the most reliable facts, we fix $1$ as the 
best confidence level (\emph{thus, higher numbers will denote lower confidence}). 
A \emph{(confidence-)annotated event atom} takes the form 
$R(\mathbf{u},\iota,\ell)$,
where $R(\mathbf{u}, \iota)$ is an event atom and 
$\ell \in \numterms^+$.

Finally, as will be detailed next, we shall employ special auxiliary predicates $\existpred$, $\termpred$,  and $\windowpred$
to define existence and termination conditions and temporal windows.

 \subsection{Specifying Events via Rules}\label{specifications}
We now introduce the syntax of rules used to define simple and meta-event predicates.

\subsubsection*{Simple Events} For simple non-persistent events, we must 
provide the conditions that allow us to infer that such the event holds (or ends) at a given timepoint, as well as defi\-ning a temporal window in order to know how to define the event intervals. 
Formally, a \emph{ruleset for a non-persistent simple event predicate} $R \in \nspreds^k$ is a set of rules consisting of: 
\begin{itemize}
\item one or more \emph{existence rules} of the form\footnote{In line with standard notations for logics for reasoning about actions, auxiliary predicates may have atoms $R(u_1, \ldots, u_n)$ of arbitrary arity as arguments.
Alternatively, we could express the same thing in classical logic 
using reification and multiple copies of the auxiliary predicates
to accommodate atoms of different arities. We write rules right-to-left as typical in logic programming. 
} 
$$\mathsf{exists}(R(u_1, \ldots, u_k), t, \ell) \leftarrow B$$
\item zero or more \emph{termination rules} of the form 
$$\mathsf{ends}(R(u_1, \ldots, u_k), t, \ell) \leftarrow B $$
\item one or more \emph{(expansion) window rules} of the form
$$\mathsf{window}(R(u_1, \ldots, u_k), w) \leftarrow B $$
\end{itemize}
where $u_1,\ldots, u_k \in \dterms$, $t \in \numterms$ (specifying a timepoint),
$\ell \in \mathbb{N}^+$
(giving the confidence level of the rule, with the convention that 1 denotes greatest confidence), and
$w \in \numterms^+$ (defining a temporal window). 
Rule bodies $B$ take the form of conjunctions
whose conjuncts may be (possibly negated) atemporal and observation atoms,
inequality atoms $z \neq z'$ between terms $z,z' \in \dterms$, 
or inequality / comparison atoms $z \bowtie z'$ for $z,z' \in \numterms$ and $\bowtie \, \in \{\neq, <, \leq\}$.
Rules are required to be \emph{safe}: every variable 
 in a rule must occur either in some unnegated atemporal or observation body atom. 
Additionally, window rules should provide a unique window value $w>0$ for each $R(c_1, \ldots, c_k)$ that exists due to the existence rules (this is made 
 formal in Section \ref{semantics}). 
The simplest way to accomplish this is to assign each predicate $R$ a fixed window, but it can be useful to be able to assign different windows based upon the event arguments. 

\begin{example}
For the simple event $\mathsf{ABTherapy}$ (shortened to $\mathsf{ABTh}$),
we could use the following rules: 
\begin{align*}
\mathsf{exists}(\mathsf{ABTh}(p,d),t,1) \leftarrow & \; \mathsf{Adm}(p,d,t) \land \mathsf{AB}(d) \\
\mathsf{ends}(\mathsf{ABTh}(p,d),t,1) \leftarrow & \; \mathsf{Stop}(p,d,t) \land \mathsf{AB}(d)\\
\mathsf{window}(\mathsf{ABTh}(p,d),48) \leftarrow & \; \mathsf{P}(p) \land \mathsf{AB}(d)
\end{align*}
The existence condition looks for administrations ($\mathsf{Adm}$) of an antibiotic ($\mathsf{AB}$) drug.
The termination rule applies when there is a drug stop notification ($\mathsf{Stop}$).
A fixed window of 48 hours is defined, but one could use multiple rules 
with diffe\-rent windows to e.g.\ differentiate by the class of antibiotics. 
\end{example}

We proceed similarly for persistent simple events, the main 
difference being that no window is needed.
A \emph{ruleset for a persistent simple event predicate} $R \in \pspreds^k$ 
thus comprises one or more existence rules and (optionally) termination rules
(having the same syntactic form as before). 

\begin{example}
Tyrosine kinase inhibitor (TKI) therapy is used in specific cases of lung cancer. It
 is intended to be taken for life unless it leads to toxicity or proves ineffective, 
 in which case a patient is switched to another TKI drug.
 We model TKI therapy ($\mathsf{TKITh}$) as a persistent simple event:
\begin{align*}
\mathsf{exists}(\mathsf{TKITh}(p,d),t,1) \leftarrow & \; \mathsf{Adm}(p,d,t) \land \text{TKI}(d) \\
\mathsf{exists}(\mathsf{TKITh}(p,d),t,2) \leftarrow & \; \mathsf{Presc}(p,d,t) \land \text{TKI}(d) \\
\mathsf{ends}(\mathsf{TKITh}(p,d),t,1) \leftarrow & \; \mathsf{Adm}(p,d',t') \land \text{TKI}(d')\\
&\; \land \text{TKI}(d) \land d' \neq d  
\end{align*}
Both administration and prescription of a TKI can be used to infer existence, but prescription is less reliable. 
\end{example}

\subsubsection*{Meta-Events} Unlike simple events, which are defined directly 
from the data, the definition of meta-event predicates may refer to simple events and other meta-events. 
A ruleset for $\metapreds$ contains rules of the following form, for $R \in \metapreds^k$:
$$R(u_1, \ldots, u_k, \iota, \ell) \leftarrow B
$$
where $u_1, \ldots, u_k \in \dterms$, $\iota \in \intterms$ (specifying the temporal interval), $\ell \in \numterms^+$ (specifying confidence level),
and the rule body $B$ may use the same kinds of conjuncts as for simple event rules, as well as
(possibly negated) confidence-annotated event atoms.
Safety of rules is defined as before, except that now 
unnegated event atoms may also be used to restrict the range of variables. 

Note that a predicate $R \in \metapreds$ may occur both 
in the head of rules and (possibly negated) in rule bodies. 
In our considered medical scenarios, we found it sufficient to consider 
\emph{stratified rulesets}, for which there exists a total preorder $\preceq$ on the predicates such that if a rule has head predicate $R'$
and body predicate $R'$ then $R' \preceq R$,
and if $R'$ occurs in negated body atom, then $R' \prec R$
 \cite{DBLP:books/mk/minker88/AptBW88}.
We thus impose this stratification 
condition, which ensures that the meta-event facts are uniquely 
determined from the simple event facts and dataset. 

\begin{example}
Negation can model exclusion conditions.
Hyperglycemia ($\mathsf{HyperGlyc}$) during pregnancy
($\mathsf{Preg}$) should not yield gestational diabetes
($\mathsf{GestDiab}$) when pre-existing diabetes ($\mathsf{PreDiab}$) is already active at onset.
We encode this by introducing an auxiliary predicate
detecting diabetes at hyperglycemia onset and excluding
such cases in the defi\-nition of $\mathsf{GestDiab}$:
\begin{align*}
&\mathsf{PreDiabAtOnset}(p,[t_1,t_2],\mathsf{min}(\ell_1,\ell_2)) \\
& \quad \leftarrow \mathsf{HyperGlyc}(p,[t_1,t_2],\ell_1) \land
\mathsf{PreDiab}(p,[t'_1,t'_2],\ell_2) \\
& \quad \quad \quad \land t'_1 \leq t_1 \land t_1 \leq t'_2\\
&\mathsf{GestDiab}(p,\intersect([t_1,t_2],[t_3,t_4]),\mathsf{min}(\ell_1,\ell_2))\\
&\quad \leftarrow \mathsf{Preg}(p,[t_1,t_2],\ell_1)
\land
\mathsf{HyperGlyc}(p,[t_3,t_4],\ell_2) \\
&\quad \quad \quad
\land \neg\,\mathsf{PreDiabAtOnset}(p,[t_3,t_4],\_)
\end{align*}
\end{example}

\subsubsection*{Constraints on Events} To enforce consistency of the
 inferred
events, we consider two types of constraints: domain-independent temporal 
constraints and domain-specific constraints. 
The fixed set of \emph{temporal constraints} $\tconstrain$
ensures that simple events having the same predicate and atemporal arguments 
(but possibly different confidence levels)
cannot have intervals that non-trivially overlap. $\tconstrain$ contains the following constraints for every
$R \in \pspreds^k \cup \nspreds^k$:
\begin{align*}
\bot \leftarrow&  R(\vect{u}, [t_1,t_2]) \wedge R(\vect{u}, [t_1',t_2']) \wedge t_1 < t_1' \wedge t_1' < t_2\\
\bot \leftarrow&  R(\vect{u}, [t_1,t_2]) \wedge R(\vect{u}, [t_1,t_2']) \wedge  t_2 \neq t_2'\\
\bot \leftarrow&  R(\vect{u}, [t_1,t_2]) \wedge R(\vect{u}, [t_1',t_2]) \wedge t_1 \neq t_1' 
\end{align*}
where $\vect{u}$ abbreviates the tuple of variables 
$u_1, \ldots, u_k$ (for distinct $u_i \in \dvars$). We assume that \emph{domain-specific constraints} (if any) have the form 
$\bot \leftarrow C$, with $C$ a conjunctive formula, whose conjuncts 
are (possibly negated) atemporal, observation, or (unannotated) event atoms, inequality or comparison atoms, subject to the usual safety condition.

\begin{example}
In a medical setting, 
we may use the following domain constraint 
enforce that a patient cannot simulta\-neously undergo two targeted therapies 
with different TKIs: 
\begin{align*}
\bot \leftarrow &~ \mathsf{TKITh}(p, d_1, [t_1, t_2]) \wedge \mathsf{TKITh}(p, d_2, [t_1', t_2']) \\
& \wedge d_1 \neq d_2 \wedge t_1 < t_1' \wedge t_1' < t_2 \wedge t_2 < t_2'
\end{align*}
\end{example}
 
 \subsubsection*{Temporal Event Specifications} We now have
all the elements needed  to define 
 temporal event specifications: 
 \begin{definition}
 A \emph{temporal event specification (TES)} takes the form 
 $\eventspec = (\simplerules, \metarules, \tconstrain, \domconstrain)$,
 where:
 \begin{itemize}
 \item $\simplerules= \bigcup_{R \in \nspreds \cup \pspreds} \Pi_R$, with $\Pi_R$ a ruleset for $R$ 
\item $\metarules$ is a ruleset for $\metapreds$
 \item $\tconstrain$ is the fixed set of temporal constraints
 \item $\domconstrain$ is a set of domain-specific constraints 
 \end{itemize}
 \end{definition}

\subsection{Semantics of the Framework}\label{semantics}
It remains to make precise which event facts are generated from 
a given TES and dataset. 
We shall 
present the semantics in
stages, starting with simple events. 

\subsubsection*{Semantics of Simple Events}
To define the $\existpred$-, \mbox{$\termpred$-,} $\windowpred$-facts that hold in a dataset~$\dataset$, we simply evaluate rule bodies (seen as first-order formulas) in $\dataset$ 
(seen as a first-order structure). 
Given a rule body $B$ over variables $\vect{u} \cup \vect{v}$,
we let $\evalcond(B, \vect{v},\dataset) = \{\vect{c} \mid \dataset \models \exists \vect{u} \, B[\vect{v}\!:\!\vect{c}]\}$,
where $B[\vect{v}\!:\!\vect{c}]$ is $B$ with variables $\vect{v}$ replaced by the constants in $\vect{c}$. 

\begin{definition}\label{sefacts}
Given a TES  $\eventspec = (\simplerules, \metarules, \tconstrain, \domconstrain)$ and dataset
$\dataset$, we define $\simplerules(\dataset)$ as the set of all
facts $\mathsf{p}(\vect{c})$ s.t.\ $\vect{c} \in \evalcond(B, \vect{v},\dataset)$ for some rule $\mathsf{p}(\vect{v}) \leftarrow B(\vect{u} \cup \vect{v})\! \in \! \simplerules$. 
For a predicate $R \in \nspreds^k$, $k$-tuple $\vect{d}$ of data constants, and confidence level $\ell$, we define:
\begin{align*}
T_\exists^\ell(R(\vect{d})) = &\{t \mid \existpred(R(\vect{d} ), t,\ell') \in \simplerules(\dataset), \ell' \leq \ell\}\\
T_\stop^\ell(R(\vect{d})) = &\{t \mid \termpred(R(\vect{d} ), t,\ell') \in \simplerules(\dataset), \ell'  \leq \ell\}
\end{align*}
We say that $\simplerules(\dataset)$ is \emph{valid} if
whenever $\existpred(R(\vect{d} ), t,\ell) \in \simplerules(\dataset)$, there is a unique 
$\windowpred(R(\vect{d} ),w) \in \simplerules(\dataset)$. 
\end{definition}

Recall that for non-persistent simple events
the idea is to construct intervals for an event by starting from the timepoints where the event is stated to exist (given by $T_\exists^\ell(R(\vect{d}))$),
then iteratively `expanding' these intervals to incorporate nearby timepoints (with `near' defined by $\windowpred(R(\vect{d} ),w)$)
stopping either when no further such timepoint is encountered, or when a termination condition ($T_\stop^\ell(R(\vect{d}))$) is reached. 
As we shall assume that $\simplerules(\dataset)$ is valid (cf.\ previous definition), it is always clear which 
window to use. The following definition formalizes this idea.

\begin{definition}\label{nspred-def}
A fact $R(\vect{d}, [t_1,t_2])$ is \emph{inferred from $(\simplerules, \dataset)$ with confidence $\ell$}, 
denoted $\simplerules, \dataset \models_\ell R(\vect{d}, [t_1,t_2])$, if: 
\begin{enumerate}
\item $\windowpred(R(\vect{d} ),w) \in \simplerules(\dataset)$
\item there exists $t_0', \ldots, t_n' \in T_\exists^\ell(R(\vect{d}))$
such that $t_0'=t_1$ and $t_{i+1}' - t_i' \leq w$ for every $0 \leq i < n$
\item there is no $t^\dag \in T_\stop^\ell(R(\vect{d}))$ such that $t_1 \leq t^\dagger < t_2$
\item for every $t_1^\sharp \in T_\exists^\ell(R(\vect{d}))$ with $t_1 - w \leq t_1^\sharp < t_1$, there exists $t_e \in T_\stop^\ell(R(\vect{d}))$ with $t_1^\sharp \leq t_e < t_1$
\item if $t_2= t_n'$, then there is no $t^\dag \in T_\exists^\ell(R(\vect{d})) \cup T_\stop^\ell(R(\vect{d}))$
with $t_n' < t^\dag \leq t_n' + w$
\item if  $t_2 \neq t_n'$, then $t_2 \in T_\stop^\ell(R(\vect{d}))$ and $t_2-t_n' \leq w$
\item there is no $\ell' < \ell$ such that $\simplerules, \dataset \models_{\ell'} R(\vect{d}, [t_1,t_2])$
\end{enumerate}
\end{definition}

We briefly explain the role of the different items of the preceding definition. Item 1 defines the (unique) time window $w$ associated with the event predicate~$R$.  
    This window determines the maximal temporal distance allowed between consecutive observations
    that belong to the same event occurrence. Item 2 ensures that there is a sequence of timepoints in $T_\exists^\ell(R(\vect{d}))$ that are sufficiently dense (adjacent timepoints are within distance $w$) 
and cover the interval $[t_1, t_n']$. Item 3 prevents the inferred interval from crossing a termination condition by requiring that no termination point lies strictly between its start and end. Items 4 and 5 
ensure that the interval could not have been extended further
in either direction 
by using an earlier or later timepoint from $T_\exists^\ell(R(\vect{d}))$. 
Item 6 makes sure that if $t_2 \neq t_n'$ then 
$t_2$ satisfies a termination condition (e.g., a drug stop notification). Finally, item 7 prevents redundant inference of the exact same event with different confidence levels.

\begin{example}
Suppose for event $E$ we have $T_\exists^1(E)= \{2,4,9\}$,
$T_\exists^2(E)= \{1,5,6,10\}$, and $T_\times^1(E)= \{7,8\}$:\vspace*{-2mm}\\ 
\begin{center}
\includegraphics[scale=0.19]{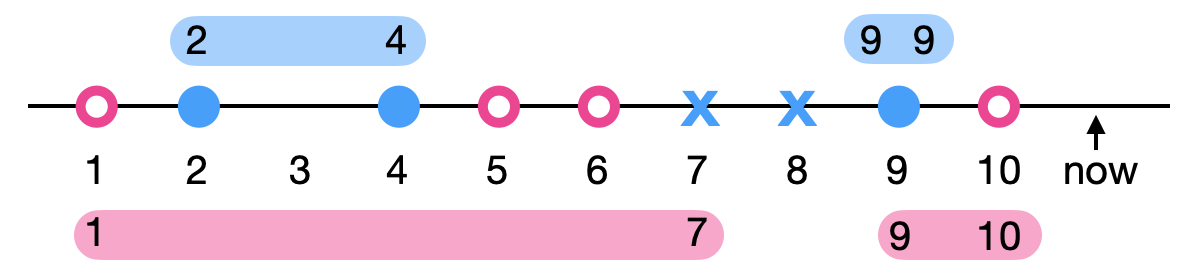}
\end{center}
In this timeline, 
each \textbf{dot} represents a timepoint in $T_\exists^\ell(E)$ where an \emph{existence condition} for $E$ is satisfied with confidence level~$\ell$ (blue dots for $\ell = 1$, pink dots for $\ell = 2$),  
while each \textbf{cross} marks a timepoint in $T_\times^\ell(E)$ 
where a \emph{termination condition} holds.
With a window of $2$ and current time $11$, we get the blue intervals 
$[2,4]$ and $[9,9]$ with confidence $1$
and the pink intervals $[1,7]$ and $[9,10]$ with confidence $2$.
\end{example}

Defining intervals for persistent simple events is simpler 
since they continue until a termination condition is reached.

\begin{definition}\label{def-pers-sem}
Consider  
$R \in \pspreds^k$ defined in 
$\simplerules$, and let $\dataset$, $\vect{d}$, $[t_1,t_2]$, $T_\exists^\ell(R(\vect{d}))$
and $T_\stop^\ell(R(\vect{d}))$ be as in Definitions~\ref{sefacts} and \ref{nspred-def}. 
Then $R(\vect{d} , [t_1,t_2])$ is \emph{inferred from $(\simplerules, \dataset)$ with confidence $\ell$}, denoted $\simplerules, \dataset \models_\ell R(\vect{d} , [t_1,t_2])$, if: 
\begin{enumerate}
\item $t_1 \in T_\exists^\ell(R(\vect{d}))$
\item if $t_2\neq \now$, then $t_2 \in T_\stop^\ell(R(\vect{d}))$
\item for every $t_1^\sharp \in T_\exists^\ell(R(\vect{d}))$ with $t_1^\sharp < t_1$,
there exists $t_e \in T_\stop^\ell(R(\vect{d}))$ with $t_1^\sharp \leq t_e < t_1$
\item if $t_2\neq \now$, there is no $t_2^\sharp \in T_\stop^\ell(R(\vect{d}))$ with $t_1 \leq t_2^\sharp < t_2$
\item if $t_2= \now$, there is no $t_2^\sharp \in T_\stop^\ell(R(\vect{d}))$ with $t_1 \leq t_2^\sharp$
\item there is no $\ell' < \ell$ such that $\simplerules, \dataset \models_{\ell'} R(\vect{d}, [t_1,t_2])$
\end{enumerate}
\end{definition}

\begin{example}
Take the preceding example but now let $E$ be a persistent event:\vspace*{-2mm}\\ 
\begin{center}
\includegraphics[scale=0.36]{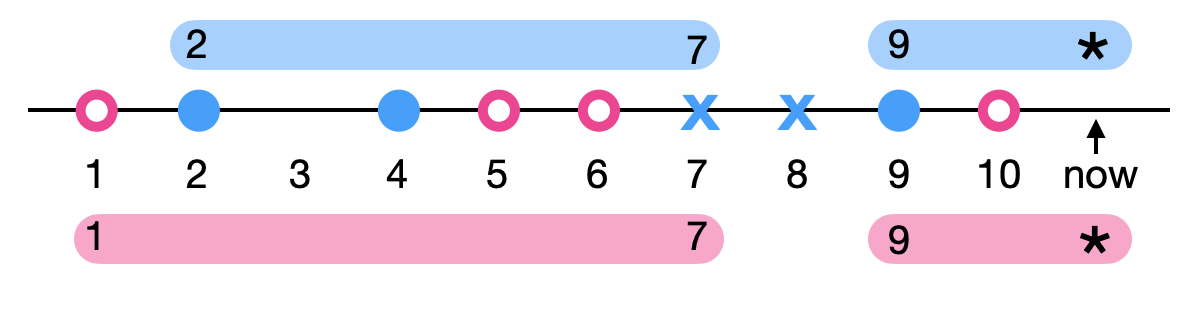}
\end{center}
The inferred intervals become $[2,7]$ and $[9,\now]$ for confidence~1, and $[1,7]$ and $[9,\now]$ for confidence~2.  
\end{example}

To facilitate later definitions, we introduce some notation for referring to the set of inferred simple event facts (with and without their associated confidence levels): 
\begin{align*}
\infevents &= \{R(\vect{d} , [t_1,t_2], \ell) \!\mid\! \simplerules, \dataset \models_\ell R(\vect{d} , [t_1,t_2])\}\\
\infeventsnoconf&= \{R(\vect{d} , [t_1,t_2]) \!\mid\! \simplerules, \dataset \models_\ell R(\vect{d} , [t_1,t_2])\}
\end{align*}
Given any set $\mathcal{S}$ of confidence-annotated facts, we let $\mathcal{S}^-$
be the result of removing the confidence levels from all facts in $\mathcal{S}$,
and let $\mathcal{S}_\ell$ be the set of facts in $\mathcal{S}$ with confidence $\ell$. 

\subsubsection*{Inferring Meta-Events}
Due to our decision to use stra\-tified rulesets to define meta-event predicates, 
there will be a unique set of meta-event facts that can be inferred from 
a dataset and a given set of simple event facts: 

\begin{definition}
Consider a dataset $\dataset$, set $\mathcal{S}$ of confidence-annnotated simple event facts,
and TES $\Sigma$ with ruleset  $\metarules$ for $\metapreds$. The set $\infermeta$ of \emph{inferred confidence-annotated meta-event facts}
contains all facts $R(\vect{d}, [t_1,t_2],\ell)$ with $R \in \metapreds$ appearing in 
the unique stratified model\footnote{Stratified models are defined by evaluating the rules accor\-ding to the ordering of the predicates \cite{DBLP:books/mk/minker88/AptBW88,DBLP:journals/csur/DantsinEGV01},
so that it is always clear how to interpret negated atoms in rule bodies.}
of $(\dataset, \mathcal{S}, \metarules)$. The set $\infermetanoconf$ is obtained from $\infermeta$
by removing the confidence levels. 
\end{definition}

\subsubsection*{Repairing Sets of Simple Events}
So far we have defined inferred events without paying attention to the constraints. Even if the underlying dataset is accurate, 
the set of inferred simple events may violate the domain constraints due to lower confidence rules, which may sometimes incorrectly suggest the existence or termination of a simple event. Additionally, the temporal constraints may be violated 
by `duplicate' events derived with different confidence levels. 

We thus propose to repair the set of inferred simple events
so that they satisfy the constraints. First, we make clear how we define consistency w.r.t.\ a TES: 

\begin{definition}\label{tes-cons}
Given a TES
$\eventspec = (\simplerules, \metarules, \tconstrain, \domconstrain)$, we say 
a set $\mathcal{S}$ of confidence-annotated simple event facts is \emph{$\eventspec$-consistent}
if $\tconstrain, \domconstrain, \dataset, \mathcal{S}^-, \mathsf{ME}^-(\dataset,\mathcal{S}, \eventspec) \not \models \bot$
(i.e.\ there is no constraint $\bot \gets C$ such that $C$
evaluates to true w.r.t.\ $\dataset \cup \mathcal{S}^- \cup \mathsf{ME}^-(\dataset,\mathcal{S}, \eventspec)$). 
\end{definition}
Observe that we need the annotated simple event facts ($\mathcal{S}$) to determine 
the associated 
meta-event facts
($\mathsf{ME}(\dataset,\mathcal{S}, \eventspec)$), but
we strip the simple and meta-event facts of their annotations when we check for constraint violations (as the 
constraints involve unannotated event atoms).  

Our first notion of repair simply considers the
inclusion-maximal consistent sets of facts, in line with the subset repairs previously studied for databases and knowledge bases, cf.\ \cite{ber,biebou}. 

\begin{definition}\label{simplerep}
Let $\mathcal{S}$ and $\eventspec$ be as in Definition \ref{tes-cons}.
Then $\mathcal{R} \subseteq \mathcal{S}$ is a \emph{repair} of $\mathcal{S}$ w.r.t.\ $\Sigma$
if (i) $\mathcal{R}$ is $\eventspec$-consistent, and (ii)
there is no $\mathcal{U} \subseteq \mathcal{S}$ such that $\mathcal{R} \subsetneq \mathcal{U}$
and $\mathcal{U}$ is $\eventspec$-consistent. We use $\reps(\mathcal{S}, \Sigma)$ for the set of repairs of $\mathcal{S}$ w.r.t.\ $\Sigma$. 
\end{definition}

\begin{example}
Consider again the running example where  
$T_\exists^1(E)= \{2,4,9\}$, $T_\exists^2(E)= \{1,5,6,10\}$, and $T_\times^1(E)= \{7,8\}$,  
yielding the intervals $[2,4]$ and $[9,9]$ (confidence~1) and  
$[1,7]$ and $[9,10]$ (confidence~2).  
Here, both pairs of intervals---$[2,4]$ and $[1,7]$, as well as $[9,9]$ and $[9,10]$---violate 
the temporal constraints because they overlap for the same event $E$.  
To restore consistency, one interval from each conflicting pair has to be selected, yielding four repairs:
\[
\begin{array}{lcl}
\mathcal{R}_1 &=& \{E([2,4],1),\,E([9,9],1)\} \\
\mathcal{R}_2 &=& \{E([2,4],1),\,E([9,10],2)\} \\
\mathcal{R}_3 &=& \{E([1,7],2),\,E([9,9],1)\} \\
\mathcal{R}_4 &=& \{E([1,7],2),\,E([9,10],2)\}
\end{array}
\]
\end{example}

We also consider preferred repairs, which preferentially retain 
facts with better confidence levels, inspired by the
 $\subseteq_P$-repairs of \cite{biebougoa}. Recall that the notation $\mathcal{R}_\ell$
 denotes the set of facts in $\mathcal{R}$ having confidence $\ell$, with $\ell=1$ giving
 the most reliable facts.

\begin{definition}\label{prefrepdef}
Let $\mathcal{S}$ and $\eventspec$ be as in Definition \ref{simplerep}, and 
let $n$ be the maximum confidence level appearing in $\mathcal{S}$. 
Then $\mathcal{R} \subseteq \mathcal{S}$ is a \emph{preferred repair} of $\mathcal{S}$ w.r.t.\ $\Sigma$
if $\mathcal{R}$ is $\eventspec$-consistent and 
there does not exist 
$\mathcal{U} \subseteq \mathcal{S}$ and $1 \leq k \leq n$ such that (i) $\mathcal{U}$ is $\eventspec$-consistent, 
(ii)~$\mathcal{U}_\ell = \mathcal{R}_\ell$ for every $1 \leq \ell < k$, and (iii)~$\mathcal{R}_k \subsetneq \mathcal{U}_k$.  We use $\prefreps(\mathcal{S}, \Sigma)$ for the set of repairs of $\mathcal{S}$ w.r.t.\ $\Sigma$. 
\end{definition}

\begin{example}
Under the preferred repair semantics, conflicts between facts 
are resolved using the confidence level of facts. 
In our running example, 
there is a unique preferred repair, $\mathcal{R}_1$,
which retains 
the two facts with confidence level~1. 
\end{example}

Note that we repair the set of simple events, using meta-events only to determine consistency, 
in order to avoid situations in which a repair contains a meta-event but the simple events 
needed to create it have been removed. 

\subsubsection*{Semantics of Temporal Event Specifications}
We are now ready to define the semantics 
of a TES and dataset: 

\begin{definition}\label{timelinedef}
Given a TES $\eventspec = (\simplerules, \metarules, \tconstrain, \domconstrain)$ and dataset $\dataset$: 
\begin{itemize}
\item the \emph{naïve timeline} is $\infevents \cup \mathsf{ME}(\dataset,\infevents, \eventspec)$ 
\item the \emph{consistent timelines} take the form $\mathcal{R} \cup \mathsf{ME}(\dataset,\mathcal{R}, \eventspec)$, where $\mathcal{R} \in \reps(\infevents, \eventspec)$
\item the \emph{preferred timelines} take the form $\mathcal{P} \cup \mathsf{ME}(\dataset,\mathcal{P}, \eventspec)$, where $\mathcal{P} \in \prefreps(\infevents, \eventspec)$
\item the \emph{cautious timeline} takes the form $\mathcal{I} \cup \mathsf{ME}(\dataset,\mathcal{I}, \eventspec)$, where $\mathcal{I} = \cap_{\mathcal{R} \in \reps(\infevents, \eventspec)} \mathcal{R}$
\end{itemize}
\end{definition}

The (unique) naïve timeline ignores the constraints and infers all annotated simple event and meta-event facts.
The consistent and preferred timelines are obtained by taking a (preferred) repair and completing it with the associated meta-event facts. 
Finally, the (unique) cautious timeline first intersects all repairs, then adds in the inferable meta-event facts. 
Note that in the absence of negated  event atoms, the naïve and cautious timelines provide upper and lower bounds, respectively, 
on the facts appearing in consistent and preferred timelines (but this does not hold in general).

\section{Complexity \& Algorithms}\label{complexity}
In this section, we examine the computational properties of our framework. 
As is common for data-centric tasks, our complexity analysis will employ
\emph{data complexity}, where only the set(s) of facts are treated as input, while the rules and constraints are treated as fixed.

In order to generate the different kinds of timelines, we first need to be
able to compute the set of inferred simple event and meta-event facts. This can be done efficiently: 

\begin{theorem}\label{thm:ptimegen}
The sets $\infevents$ and $ \mathsf{ME}(\dataset,\mathcal{S}, \eventspec)$ can be computed in \textsc{PTime} in data complexity. 
\end{theorem}
\begin{proof}[Proof sketch]
Computing $\simplerules(\dataset)$ 
essentially corresponds to evaluation of first-order queries over a database, 
a task 
well known to be computable in (sub)polynomial time\footnote{We direct interested readers to Chapter 17.1 of \cite{DBLP:books/aw/AbiteboulHV95} for more information on the complexity of first-order query evaluation and a proof of membership in $\mathsf{AC}_0$. } in data complexity (more precisely, in $\mathsf{AC}_0 \subseteq \textsc{LogSpace}$). 
Likewise, 
$ \mathsf{ME}(\dataset,\mathcal{S}, \eventspec)$ is \textsc{PTime}-computable once we have already computed $\infevents$, as this essentially 
corresponds to evaluating a stratified Datalog program, which has \textsc{PTime} data complexity, cf.\ \cite{DBLP:journals/csur/DantsinEGV01}. 

It therefore only remains to explain how to compute $\infevents$ from the sets $T_\exists^\ell(R(\vect{d}))$ and $T_\stop^\ell(R(\vect{d}))$ and the window provided by $\windowpred(R(\vect{d} ),w) \in \simplerules(\dataset)$. A naïve yet polynomial-time procedure would consider all quadra\-tically many possible intervals $[t,t']$ for a given $R(\vect{d})$ and check whether each condition of Definition \ref{nspred-def} (resp.\ Definition \ref{def-pers-sem} for persistent events) is satisfied. Naturally, one can devise more efficient algorithms that consider fewer candidate intervals
(we describe in the appendix how simple events are computed in our system). \end{proof}

It follows that the naïve timeline can be efficiently computed. 
By contrast, there could be exponentially many different (preferred) repairs, 
so it may not be feasible to compute all consistent and preferred timelines. 
In fact, by suita\-bly adapting complexity results for atemporal repairs, 
we can show it is intractable even to recognize such timelines:

\begin{theorem}\label{thm:conpgen}
It is \textsc{coNP}-complete in data complexity to decide, given a TES $\eventspec$, dataset $\dataset$, and set of facts $\mathcal{S}$,
whether $\mathcal{S}$ is a consistent (or preferred) timeline for $\eventspec, \dataset$. 
\end{theorem}
\begin{proof}[Proof sketch]
To establish the \textsc{coNP} upper bound for consistent timelines, consider the following guess-and-check procedure, whose input is 
a TES $\eventspec$, dataset $\dataset$, and set $\mathcal{S}$ of simple event and meta-event facts:
\begin{enumerate}
\item Compute $\infevents$ and $\mathcal{S}_{\mathsf{SE}} = \mathcal{S} \cap \infevents$. 
\item Guess a subset $\mathcal{S}'_{\mathsf{SE}}$ of $\infevents$.
\item Check if $\mathcal{S}_{\mathsf{SE}}$ and $\mathcal{S}'_{\mathsf{SE}}$ are $\eventspec$-consistent. 
\item Return `yes' if one of the following holds (else 
 `no'): 
\begin{enumerate}
\item $\mathcal{S} \neq \mathcal{S}_{\mathsf{SE}} \cup \mathsf{ME}(\dataset,\mathcal{S}_{\mathsf{SE}}, \eventspec)$
\item $\mathcal{S}_{\mathsf{SE}}$ is \emph{not} $\eventspec$-consistent, or
\item  $\mathcal{S}'_{\mathsf{SE}}$ is $\eventspec$-consistent and $\mathcal{S}_{\mathsf{SE}} \subsetneq \mathcal{S}'_{\mathsf{SE}}$ 
\end{enumerate}
\end{enumerate}
It can be shown that some execution of this non-deterministic procedure returns `yes' iff $\mathcal{S}$ is \emph{not} a consistent timeline. 
Moreover, the procedure is easily adapted to preferred repairs by replacing (c) by the following condition: 
$\mathcal{S}'_{\mathsf{SE}}$ is $\eventspec$-consistent and there exists 
$k$ such that
(i)~$(\mathcal{S}'_{\mathsf{SE}})_\ell = (\mathcal{S}_{\mathsf{SE}})_\ell$ for every $1 \leq \ell < k$, 
and (ii)~
$(\mathcal{S}_{\mathsf{SE}})_k  \subsetneq (\mathcal{S}'_{\mathsf{SE}})_k $. \medskip

For the lower bound, we reduce 3SAT to the problem of tes\-ting whether a set of facts is \emph{not} 
a consistent timeline. 
Consider a propositional 3CNF $\varphi = \lambda_1 \wedge \ldots \wedge \lambda_m$ over 
variables $v_1, \ldots, v_k$,
where $\lambda_i= l_{i,1} \vee l_{i,2} \vee l_{i,3} $. 
We associate with each clause $\lambda_i$ a corresponding vector $(v_{i,1}, b_{i,1}, v_{i,2}, b_{i,2}, v_{i,3}, b_{i,3})$
where $v_{i,j}$ is the variable in lite\-ral $l_{i,j}$ and $b_{i,j}=1$ (resp.\ $b_{i,j}=0$) if $l_{i,j} = v_{i,j}$ (resp.\ $l_{i,j} = \neg v_{i,j}$). 
We 
encode $\varphi$ using the following 
 dataset $\dataset_\varphi$:
\begin{align*}
\dataset_\varphi = & \{\mathsf{Var}(v_i, 0) \mid 1 \leq  i \leq k\}\\
& \cup \{\mathsf{Clause}(v_{i,1}, b_{i,1}, v_{i,2}, b_{i,2}, v_{i,3}, b_{i,3}) \mid 1 \leq i \leq m  \}
\end{align*}
We define a TES $\eventspec= (\simplerules, \emptyset, \tconstrain, \domconstrain)$, 
which 
uses two simple persistent events $\mathsf{Q}$ (arity 0) and $\mathsf{Value}$ (arity 2).
The set $\simplerules$ consists of the following
 three existence rules: 
\begin{align*}
\existpred(\mathsf{Q}, t, 1) \gets & \mathsf{Var}(x,t)\\
\existpred(\mathsf{Value}(x,1), t, 1) \gets & \mathsf{Var}(x,t)\\
\existpred(\mathsf{Value}(x,0), t, 1) \gets & \mathsf{Var}(x,t)
\end{align*}
The set $ \domconstrain$ 
contains 
the following three constraints:
\begin{align*}
\gets\,\, & \mathsf{Value}(x,1,[t,t']) \wedge \mathsf{Value}(x,0,[t,t'])\\
\gets\,\, &  \mathsf{Var}(x,t) \wedge \mathsf{Q}([t,t']) \\
& \wedge \neg \mathsf{Value}(x,1,[t,t']) \wedge \neg \mathsf{Value}(x,0,[t,t'])      \\
\gets\,\, &  \mathsf{Clause}(x_1,y_1,x_2,y_2, x_3,y_3) \wedge \mathsf{Q}([t,t'])\\ 
&\wedge \neg \mathsf{Value}(x_1,y_1,[t,t']) \wedge \neg \mathsf{Value}(x_2,y_2,[t,t']) \\
& \wedge \neg \mathsf{Value}(x_3,y_3,[t,t']) 
\end{align*}
Importantly, $\eventspec$ does not depend on the instance $\varphi$, 
as required for a data complexity reduction.
It is easy to see that: 
\begin{align*}
\mathsf{SE}(\dataset_\varphi, \eventspec)= & \{\mathsf{Q}([0,\now],1)\} \cup\\ 
& \{\mathsf{Value}(v_i,b, [0,\now],1) \mid b \in \{0,1\}, 1 \leq i \leq k \}
\end{align*}
To complete the proof, one can verify that  
 $\{ \mathsf{Q}([0,\now], 1) \}$ is \emph{not} a consistent timeline of ($\eventspec,\dataset_\varphi$)  
iff $\varphi$ is satisfiable.
\end{proof}

The reduction used to show \textsc{coNP}-hardness employed a TES with negated event atoms. 
We show this is necessary, as the recognition problems are tractable if we disallow negated event atoms in rules and constraints
(note that negation can still be applied to the atemporal and observation atoms). 

\begin{theorem}\label{thm:ptime-noneg}
It can be decided in \textsc{PTime} in data complexity whether a set of facts $\mathcal{S}$ is a consistent (or preferred) timeline for 
a TES $\eventspec$ \emph{without negated event atoms} and dataset $\dataset$.
\end{theorem}
\begin{proof}[Proof idea]
The key to obtaining tractability is to show that $\Sigma$-inconsistency is monotonic, which implies that if 
$\mathcal{S} \subseteq \infevents$ is $\Sigma$-consistent and \emph{not} a repair, then there exists $\varphi \in \infevents \setminus \mathcal{S}$ that can be added while retaining consistency. It thus suffices to iterate over all such $\varphi$ and perform a $\Sigma$-consistency check to determine (non-)maximality of the candidate consistent timeline (in line with procedures for recognizing subset repairs, cf.\ Lemma 1 of \cite{biebou}). 
A similar but slightly more complex strategy can be employed for preferred timelines. 
\end{proof}

For the cautious timeline, however, the absence of negated event atoms does not suffice to 
ensure tractability.

\begin{theorem}\label{hardness-cautious}
It is \textsc{coNP}-hard in data complexity to recognize or compute the cautious timeline, even in the absence of negated event atoms. 
\end{theorem}
\begin{proof}[Proof idea]
We again proceed by reduction from 3SAT, adapting the proof of Theorem \ref{thm:conpgen}. We modify $\dataset_\varphi$ by adding a constant $c_i$ to the $\mathsf{Clause}$ fact encoding the $i$th clause and add atemporal facts $\mathsf{First}(c_1), \mathsf{Last}(c_m)$, and $\mathsf{Next}(c_i, c_{i+1})$ ($ 1 \leq i < m $). We keep the same set $\simplerules$ and retain the constraint that enforces a single truth value (0 or 1) per variable. We add six meta-rules which serve to derive $\mathsf{Sat}(c_i,[0,\now])$  if the truth assignment selected via the $\mathsf{Value}$ facts makes $c_1, \ldots, c_i$ hold (using the $\mathsf{First}$ and $\mathsf{Next}$ facts to `iterate' over the clauses). Finally, a second (negation-free) domain constraint 
$\gets\,\,  \mathsf{Q}([t,t']) \wedge \mathsf{Last}(z) \wedge \mathsf{Sat}(z,[t,t'])$
ensures $\varphi$ is satisfiable iff $\{\mathsf{Q}([0,\now],1)\}$ is \emph{not} the cautious timeline. 
\end{proof}

Interestingly, however, if we consider the special case in which we
only have the fixed set of temporal 
constraints (i.e. no domain constraints) and all termination rules have 
the same confidence, then there is a unique preferred repair, which 
moreover is efficiently computable: 

\begin{theorem}\label{thm:onlytemp}
When $\domconstrain= \emptyset$ and termination rules all have confidence $1$, 
there is a unique preferred repair, 
and both the preferred timeline and cautious timeline can be computed in \textsc{PTime} in data complexity. 
\end{theorem}
\begin{proof}[Proof sketch]
When $\domconstrain= \emptyset$, the cautious timeline can be computed in \textsc{PTime} by (i) removing those $R(\vect{u}, [t_1,t_2])$ from $\infevents$ such that there exists some  $R(\vect{u}, [t_1',t_2']) \in \infevents$ with $[t_1',t_2'] \neq [t_1,t_2]$
where 
$[t_1,t_2]$ and
$[t_1',t_2']$ non-trivially overlap, then (ii) applying the meta-event rules. 

The \textsc{PTime} result for preferred timelines is obtained by analyzing how inferred intervals are related. Indeed, when $\domconstrain= \emptyset$ and all termination rules have confidence 1, we can show that if distinct
$R(\vect{d} , [t_1,t_2])$ and $R(\vect{d} , [t_1',t_2'])$ are inferred at the same confidence level, then $[t_1,t_2]$ and $[t_1',t_2']$ cannot overlap, so no repair is needed within a single confidence level. Moreover, 
if $\simplerules, \dataset \models_\ell R(\vect{d} , [t_1,t_2])$
and $\simplerules, \dataset \models_{\ell'} R(\vect{d} , [t_1',t_2'])$ with $\ell' > \ell$, then $[t_1,t_2]$ and $[t_1',t_2']$ can only overlap if $[t_1',t_2']$ fully contains $[t_1,t_2]$. This implies uniqueness and allows us to greedily build a repair level by level, as formalized in Algorithm \ref{alg:theorem3_preferred_timeline}.
\end{proof}

\begin{algorithm}[t!]
\caption{Preferred timeline (Theorem~5)
}
\label{alg:theorem3_preferred_timeline}
\begin{algorithmic}[1] 
\REQUIRE TES $\eventspec=(\simplerules,\metarules,\tconstrain,\domconstrain)$ s.t.\
$\domconstrain=\emptyset$ and all termination rules 
have confidence~$1$, dataset $\dataset$
\ENSURE 
unique preferred timeline $\mathcal{T}^*$ 
\STATE 
$\mathcal{S} \gets \infevents$  \hfill // inferred simple events // 
\STATE  $m \gets \min\{\ell \mid R(\mathbf{d},[t_1,t_2],\ell)\in\mathcal{S}\}$
\STATE $n \gets \max\{\ell \mid R(\mathbf{d},[t_1,t_2],\ell)\in\mathcal{S}\}$
\STATE Partition $\mathcal{S}$ into $\mathcal{S}_m,\dots,\mathcal{S}_n$ 
by confidence level
\STATE $\mathcal{R}^* \gets \mathcal{S}_m$  \hfill // initialize with top-confidence facts //

\FOR{$\ell \gets m+1$ \TO $n$}
    \FORALL{$\varphi \in \mathcal{S}_\ell$}
        \IF{\textsc{NoTemporalConflict}$(\varphi,\mathcal{R}^*) 
        $}
            \STATE $\mathcal{R}^* \gets \mathcal{R}^* \cup \{\varphi\}$
        \ENDIF
    \ENDFOR
\ENDFOR

\STATE $\mathcal{T}^* \gets \mathcal{R}^* \cup \mathsf{ME}(\dataset,\mathcal{R}^*,\eventspec)$
\STATE \textbf{return} $\mathcal{T}^*$ 
\; 
\end{algorithmic}
\end{algorithm}


\section{HEVA System}\label{casper}
To evaluate the interest of our proposed approach,
we implemented core components of our framework 
using answer set programming (ASP), a prominent declarative programming paradigm\footnote{We assume basic familiarity with ASP, see e.g.~\cite{DBLP:journals/cacm/BrewkaET11,asp-in-prac-2012,DBLP:books/sp/Lifschitz19}}. 
Our prototype system \casper
(\textbf{H}igh-level \textbf{Ev}ents with \textbf{A}SP) currently only supports 
temporal constraints (i.e.\ $\domconstrain= \emptyset$). 
The system accepts termination rules with multiple confidence levels, but the computation of
preferred repairs uses 
Algorithm~\ref{alg:theorem3_preferred_timeline}, 
which requires that termination rules have confidence level~1.
The source code, documentation, and examples are publicly available on GitHub (\url{https://github.com/yvoawk/HEVA}).

\subsubsection*{System Inputs} \casper\ takes three forms of input:  event rules, atemporal facts, and observation facts.
The event rules of the TES ($\simplerules, \metarules$) are specified as ASP rules 
using head predicates \texttt{exists}, \texttt{exists\_pers} (existence conditions for persistent events), \texttt{terminates}, \texttt{pt\_window} (provided time window), and \texttt{m\_event}.
Atemporal facts are encoded as usual ASP facts, e.g.\ \texttt{tki}(\texttt{ceritinib}),
while observation facts are encoded using the \texttt{obs} predicate, e.g.\ $\texttt{obs}(\texttt{has\_adm}, \texttt{p}, \texttt{d}, \texttt{t})$ (note the reified predicate name). In our experiments, these facts were generated by an external Python script using a mapping file that links \texttt{obs} predicates to fields in a relational database.

\subsubsection*{System Components} \casper\ 
is composed of five interac\-ting ASP modules.
The \emph{non-persistent simple event module} computes non-persistent simple events facts from the input facts and \texttt{exists} and \texttt{terminates} rules, by expanding intervals iteratively and pruning non-maximal intervals.
The \emph{persistent simple event module} computes persistent simple events facts from the input facts and \texttt{exists\_pers} and \texttt{terminates} rules.
 The \emph{temporal predicate module} defines standard temporal relations between intervals, like Allen interval relations \cite{allen_Maintaining_1983}, interval manipulation predicates (e.g.\ to compute the $\mathsf{intersect}$ function), and further helper predicates that simplify rule writing. 
These defined predicates can be used in 
meta-event rule bodies, while additional helper predicates, internal to \casper, are provided by the \emph{auxiliary module}.
Finally, the \emph{temporal repair module} computes the set of repairs or the unique preferred repair. 

\subsubsection*{System Functionalities} Based upon the options that have been selected by the user, 
\casper\ passes the required ASP programs and facts to the Clingo\footnote{\url{https://potassco.org/clingo/}} ASP system \cite{DBLP:journals/tplp/GebserKKS19}, which produces \emph{stable model(s)} corresponding to the desired timeline(s).
By default, it returns a single answer set (naïve timeline), as the repair mechanism is disabled. 
When the \texttt{repair} option is enabled, \casper\ may return multiple answer sets, each corresponding to a consistent timeline, 
or the (unique) preferred or cautious timelines, 
if the \texttt{preferred} or \texttt{cautious} mode is specified. 
Note that the restriction to termination rules of confidence 1 is only 
required for the \texttt{preferred} mode.


\section{Experimental Evaluation}\label{eval}
We evaluate our approach on a medical use case based on real clinical data, focusing on both computational performance and the quality of the inferred events.

\subsection{Lung Cancer Use Case}
We formalized a lung cancer use case within our logical framework\footnote{Further details on the modelling of this use case, including the event rules, are can be found in the appendix.}. The objective was to identify six clinical events of interest.
Five were modelled as simple events: 
(a) \emph{primary lung cancer episode}, inferred from diagnostic codes (ADICAP~
and ICD-10
), ordered by confidence level $\ell$: ADI\-CAP codes ($\ell$ = 1), specific ICD-10 codes ($\ell$ = 2), and non‐specific ICD-10 codes ($\ell$ = 3), (b) \emph{secondary
  cancer episode}, inferred from ICD-10 codes, (c) \emph{EGFR and ALK mutations}, both inferred from DNA sequencing results, 
    and (d) \emph{TKI therapy}, inferred from administration ($\ell$ = 1) or prescription records ($\ell$ = 2). 
The sixth event, \emph{lung cancer disease}, is a meta-event, inferred using the primary and secondary episode events. 
The formalization involved 16 event rules with 3 confidence levels and 497 atemporal facts.

Observation facts were extracted from the clinical data warehouse of Bordeaux University Hospital. For this 
use case, we focus on the 322 patients with EGFR- or ALK-mutated lung cancer treated with TKIs, selected from approx.\ 
16,800 lung cancer cases. Table~\ref{execution_stats_lung_cancer} gives 
statistics on the number of observation facts and ground rules per patient.

\subsection{System Performance}
All experiments were run on a machine equipped with an 12th Gen Intel(R) Core(TM) i3-12100T @2.20GHz ×4, 8GB RAM, under Windows 10 Professional 64 bits, with runtimes  averaged over 5 executions.
 This use case involves rules with different confidence levels, and \casper\ was therefore run in its different modes to generate the four kinds of timeline.
As runtime results across the different modes were broadly similar, we only report results for the lung cancer study in the \texttt{repair} mode (consistent timelines)\footnote{Results for the other modes are provided in the appendix.}.
Table~\ref{execution_stats_lung_cancer} provides statistics on runtime and the number of stable mo\-dels (corresponding to consistent timelines) for each of the 322 patients in the lung cancer use case.

    \begin{table}[t]
    \setlength{\tabcolsep}{1.2mm}
    \centering
\begin{threeparttable}
        \caption{
        Statistics for lung cancer use case (322 patients). 
        Execution times for computing all stable models for consistent timelines. Q1, Q2, and Q3 indicate first, second (median), and third quartiles. } 
        \begin{tabular}{lcccccc}
        & Min. & Q1 & Q2 & Mean & Q3 & Max.\\\midrule
       \textbf{Obs.\ facts} & 3.00 & 23.00 & 45.00 & 83.65 & 99.50 & 815.00\\
\textbf{Grnd.\ rules} & 538 & 708 & 841.5 & 983.9 & 1,066.5 & 3,918\\
\textbf{Models} & 1.00 & 2.00 & 2.00 & 3.42 & 4.00 & 24.00 \\
\textbf{Time (s)} & 0.14 & 0.21 & 0.31 & 0.37 & 0.41 & 4.92 \\ \bottomrule
        \end{tabular}
    \label{execution_stats_lung_cancer}
   \end{threeparttable}
    \end{table}

The number of input observation facts varied significantly across patients, reflecting the heterogeneity of patient histories, 
which also led to variability in the number of ground rules and the execution times.  
This meant that while execution time remained low for most patients, typically completing in less than half a second, higher runtimes (up to 5 seconds) were observed for the few `outlier' patients.
The vast majority of cases (295 out of 322) triggered \ncasper's repair mechanism to resolve temporal constraint violations, 
which led to an average of 3.4 (up to a maximum of 24) stable models (
$\sim$ consistent timelines) per patient.

\subsection{Qualitative Study} 
To gain insights into the quality of the inferred events,
we compared the events from \ncasper’s
\emph{consistent timelines} against the individual annotations generated by four experts. 
We randomly selected 30 patients  from the lung cancer cohort
from those with a number of observations between the median 
and third quartile.
Each expert manually examined 15 patient records (giving two sub-cohorts: Annotators~1 \&~2, Annotators~3 \&~4) 
to identify the target clinical events and to indicate
their start date and, when applicable, their end date.
An inter-annotator agreement\footnote{Details on the agreement score computation and a breakdown of the results 
are provided in the appendix.} score was computed using a component-based weighted scheme that accounts for event structure, with higher weights assigned to events requiring finer-grained annotation. 
Figure~\ref{fig:qualitative_overview} summarizes 
the inter-annotator agreement between the two experts handling the same
sub-cohort, as well as 
the mean
agreement ratios between \casper and each annotator.

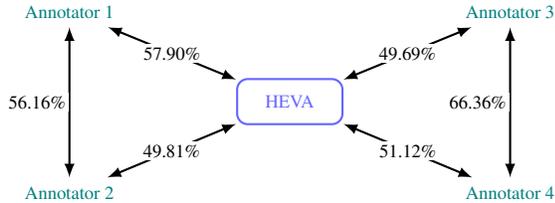
\begin{figure}[t!]
\centering
\begin{tikzpicture}[
    node distance=1.4cm,
    font=\scriptsize,
    casper/.style={
        draw,
        rectangle,
        rounded corners,
        align=center,
        minimum width=1.4cm,
        minimum height=0.6cm,
        thick
    },
    annotator/.style={
        align=center
    },
    edge label/.style={
        fill=white,
        inner sep=1pt,
        font=\scriptsize
    },
    every edge/.style={draw, <->, thick, >=latex}
]

\node[casper, blue!60] (casper) {HEVA};

\node[annotator] (a1) [left=1.5cm of casper, yshift=1.2cm, teal] {Annotator 1};
\node[annotator] (a2) [below=2cm of a1, teal] {Annotator 2};
\node[annotator] (a3) [right=1.5cm of casper, yshift=1.2cm, teal] {Annotator 3};
\node[annotator] (a4) [below=2cm of a3, teal] {Annotator 4};

\draw (a1) edge node[pos=0.5, edge label] {57.90\%} (casper);

\draw (a2) edge node[pos=0.5, edge label] {49.81\%} (casper);

\draw (a3) edge node[pos=0.5, edge label] {49.69\%} (casper);

\draw (a4) edge node[pos=0.5, edge label] {51.12\%} (casper);

\draw (a1) edge[left=20]
    node[pos=0.5, edge label] {56.16\%} (a2);

\draw (a3) edge[left=20]
    node[pos=0.5, edge label] {66.36\%} (a4);

\end{tikzpicture}
\caption{Mean agreement between HEVA (consistent timelines) and annotators,
with inter-annotator agreement per sub-cohort.}
\label{fig:qualitative_overview}
\end{figure}

Overall inter-annotator agreement among medical experts averaged 61\%, reflecting substantial variability, particularly for temporal boundaries (start/end dates), highlighting the intrinsic difficulty of this task 
even for human experts. Comparing the expert annotations, we found that agreement was consistently high for event presence across all categories, while start dates showed the lowest concordance, highlighting the inherent ambiguity of temporal information in EHRs. Primary lung cancer episodes and lung cancer di\-sease achieved the highest agreement, whereas secondary cancer episodes and ALK mutations showed lower consistency, likely due to clinical complexity and data sparsity. 

We compared \ncasper’s inferred events against 
expert annotations using the same scoring scheme. As each consistent timeline gives a plausible interpretation of the data,  
for the evaluation, we retained the consistent timeline that achieved the highest agreement with the corresponding expert annotations. 
Agreement between \casper and experts ranged from 49\% to 58\% per annotator, rising to 60\% when merging annotator pairs. This is close to the inter-annotator agreement, indicating that \ncasper’s outputs are roughly as consistent with expert judgments as experts amongst themselves. 

\section{Related Work}\label{related}
We briefly review approaches to temporal reasoning that are 
closest to our own in terms of motivations or methods. 

The original event calculus (EC), introduced for reasoning about actions and their effects \cite{DBLP:journals/ngc/KowalskiS86}, considers instantaneous events (
akin to our observations) that initiate or terminate fluents (properties whose value may change over time), which persist by default through inertia until terminated. Domain modeling in EC is done via rules for specifying initiation or termination of fluents where rule heads use special predicates \texttt{initiatesAt} and \texttt{terminatesAt}, and rule bodies speak of which other events/actions and fluents (do not) hold at the considered timepoint. Our modeling of persistent simple events is broadly similar to the handling of fluents in EC (but using existence rules rather than initiation rules to determine interval start times and with different restrictions on rule body syntax). By contrast, our formalization of non-persistent simple events via existence, termination, and window rules, equipped with a ``group close existence points together" semantics, has no direct analog in any EC dialects. 
Another common point with the EC is the use of special predicates which avoids the need for temporal logic operators in rule bodies, arguably leading to simpler and more intuitive specifications for domain experts.

Several extensions of the EC, in particular the runtime event calculus (RTEC), have been subsequently developed for \emph{complex (aka composite) event recognition (CER)} \cite{artikis_Event_2015,DBLP:journals/jair/MantenoglouPA25}, with an emphasis on performance to enable real-time processing of streaming data. Some such EC dialects (like RTEC) adopt an interval-based semantics for fluents (akin to our event predicates) and additionally allow for rules to define complex events in terms of other complex events using interval manipulation (intersecting or unioning event intervals) or Allen relations \cite{DBLP:conf/kr/MantenoglouKA23}. Our meta-event rules likewise support hierarchical and compositional modeling of complex events. 

The need to handle various kinds of uncertainty in CER is widely recognized, cf.\ survey by \citeauthor{DBLP:journals/csur/AlevizosSAP17} (\citeyear{DBLP:journals/csur/AlevizosSAP17}), motivating 
the development of probabilistic CER frameworks, including EC-based ones, for 
handling uncertain data, where probabilities are attached to the timestamped facts. Computing the probability of a complex event is computationally challenging and has been recently tackled 
using linear algebraic methods \cite{DBLP:conf/kr/TsilionisAP25}. Uncertainty in event inference (called pattern uncertainty by \citeauthor{DBLP:journals/csur/AlevizosSAP17}) is less explored. Moreover, to 
the best of our knowledge, qualitative approaches to handling uncertainty and inconsistency (like our 
confidence levels and repair-based timeline semantics) have not yet been considered for CER, nor is it evident how they could be simulated using existing probabilistic CER methods. 


Several other rule-based formalisms have been proposed for reasoning over temporal data. DatalogMTL \cite{DBLP:journals/jair/BrandtKRXZ18,DBLP:conf/rweb/Walega25} extends Datalog with metric temporal operators to define complex temporal queries. 
Due to its expressivity, reasoning is highly intractable (\textsc{PSpace} data complexity even for the integer timeline), 
though relevant fragments with lower complexity have been identified and implemented \cite{DBLP:conf/ijcai/WalegaGKK20,DBLP:conf/kr/WalegaGKK20,DBLP:conf/aaai/WangHWG22}, with recent support for streaming data \cite{DBLP:journals/ws/WalegaKWG23}. 
LARS is an expressive rule-based language specifically designed for reasoning over streaming data \cite{DBLP:journals/ai/BeckDE18}, with a dedicated window operator to restrict 
to recent timepoints or atoms, and a recent extension 
to support ontological reasoning \cite{DBLP:conf/kr/UrbaniKE22}. Extensions of ASP with linear and metric temporal operators have also been explored, implemented in the ASP tool telingo \cite{DBLP:conf/lpnmr/CabalarKMS19,cabalar_Temporal_2022}. Differently from these 
works, our existence and termination rule bodies 
do not utilize temporal operators but only conditions close to relational queries with which medical informatics practitioners are typically familiar.  

In the medical domain, the need for presenting temporal information at different levels of abstraction by combining timestamped observations has been long acknow\-ledged \cite{DBLP:journals/ai/Shahar97,DBLP:journals/artmed/ShaharM96}.
Rule-based approaches 
are desirable as they support easy integration of domain knowledge and explainability.  
The advantages of adopting declarative approaches were highlighted in the works of
Falcionelli \textit{et al.}~(\citeyear{falcionelli_Indexing_2019}) who employ an EC-based CER approach to monitor chronic conditions from sensor data,
and 
Dwyer \textit{et al.}~(\citeyear{dwyer_Reasoning_2023}) who employ the Vadalog rule language to extract and analyze patient pathways from EHRs
(neither work 
considers 
consistency handling). 
Temporal rules have also been successfully used to implement domain-specific algorithms, 
as exemplified by the work of Lyu \textit{et al.}~(\citeyear{lyu_Temporal_2022}) on gestational age detection from EHR data.


\section{Conclusion and Future Work}\label{conclusion}
In this paper, we introduced an expressive logical framework for inferring temporally extended events. 
At the heart of our approach is a novel method for specifying simple events through the use of existence and termination conditions and
temporal windows, without the need to write rules with temporal logic operators. Another distinguishing feature is the use of confidence levels, constraints, and a repair mechanism to define different kinds of (preferred) timelines, accounting for the inherent uncertainty in the event detection process (which is a difficult task even for human experts).
This unique combination of language features required us to conduct a new complexity analysis, leading to the identification of relevant special cases with more favorable computational properties, and it also meant that we could not straightforwardly implement our framework on top of existing temporal reasoning systems. 
We therefore developed 
an ASP-based prototype \casper, which showed promising results on a lung cancer use case.

Our decision to use ASP for the implementation was motivated not only by its ease of use for prototyping, but also
by its high expressive power, which will be useful when extending \casper\ to handle arbitrary TESs and new reasoning tasks (e.g.\ temporal query answering over the generated timelines). 
Indeed, in order to handle repairs w.r.t.\ arbitrary constraints, we hope to 
leverage
exis\-ting work on ASP-based repair techniques~\cite{DBLP:journals/tods/EiterFGL08,DBLP:journals/tplp/MannaRT13,KR26-BBJM}. 
It would also be relevant to conduct a detailed expressivity study in order to understand precisely which fragments of our language
can be captured by existing temporal formalisms  
(we expect for instance that simple event inference could be reduced to reasoning in suitably chosen EC and DatalogMTL dialects). 
Such expressivity results may suggest new ideas for optimizing timeline computation or for adapting our framework to handle streaming data. 


While our logical framework is application-independent, it was developed with clinical event detection in mind.
To facilitate its use  
 by medical practitio\-ners, 
we plan to develop a domain-specific language offe\-ring users a simplified syntax and templates covering common medical event types. We also wish to explore how system-generated explanations could 
help domain experts better understand, validate, and potentially revise their own annotations and/or support iterative refinement of the rules. 
The (semi-)automatic generation of domain constraints and atemporal facts
from medical ontologies is another interesting direction. 

\section*{Acknowledgements} This work was partially supported by the ANR AI Chair INTENDED (ANR-19-CHIA-0014) 
and JST CREST Grant Number JPMJCR22D3. The authors would also like to acknowledge Frantz Thiessard, Antoine Lanusse, Guillaume Verdy, Léodoric Ahouanse and Arslane Tedlaouti for their valuable contributions and support.

\bibliographystyle{kr}
\bibliography{refs}

@inproceedings{KR26-BBJM,
  author       = {Meghyn Bienvenu and Camille Bourgaux and Robin Jean and Giuseppe Mazzotta},
  title        = {Using {ASP(Q)} to Handle Inconsistent Prioritized Data},
  booktitle    = {Proc. of KR},
  year         = {2026}
}

@inproceedings{DBLP:conf/rweb/Walega25,
  author       = {Przemyslaw Andrzej Walega},
  title        = {Reasoning About Time in {DatalogMTL}: Course Notes},
  booktitle    = {Proc. of Reasoning Web Summer School},
  series       = {OASIcs},
  pages        = {9:1--9:23},
  year         = {2025}
}

@article{DBLP:journals/ws/WalegaKWG23,
  author       = {Przemyslaw Andrzej Walega and
                  Mark Kaminski and
                  Dingmin Wang and
                  Bernardo Cuenca Grau},
  title        = {Stream reasoning with {DatalogMTL}},
  journal      = {J. Web Semant.},
  volume       = {76},
  pages        = {100776},
  year         = {2023}
}

@inproceedings{DBLP:conf/kr/WalegaGKK20,
  author       = {Przemyslaw Andrzej Walega and
                  Bernardo Cuenca Grau and
                  Mark Kaminski and
                  Egor V. Kostylev},
  title        = {{DatalogMTL} over the Integer Timeline},
  booktitle    = {Proc. of KR},
  pages        = {768--777},
  year         = {2020}
}

@article{DBLP:journals/jair/MantenoglouPA25,
  author       = {Periklis Mantenoglou and
                  Manolis Pitsikalis and
                  Alexander Artikis},
  title        = {Reasoning over Streams of Events with Delayed Effects},
  journal      = {J. Artif. Intell. Res.},
  volume       = {84},
  year         = {2025}
}

@inproceedings{DBLP:conf/kr/UrbaniKE22,
  author       = {Jacopo Urbani and
                  Markus Kr{\"{o}}tzsch and
                  Thomas Eiter},
  title        = {Chasing Streams with Existential Rules},
  booktitle    = {Proc. of KR},
  year         = {2022}
}

@article{DBLP:journals/ai/BeckDE18,
  author       = {Harald Beck and
                  Minh Dao{-}Tran and
                  Thomas Eiter},
  title        = {{LARS:} {A} Logic-based framework for Analytic Reasoning over Streams},
  journal      = {Artif. Intell.},
  volume       = {261},
  pages        = {16--70},
  year         = {2018}
}

@inproceedings{DBLP:conf/kr/TsilionisAP25,
  author       = {Efthimis Tsilionis and
                  Alexander Artikis and
                  Georgios Paliouras},
  title        = {A Tensor-Based Probabilistic Event Calculus},
  booktitle    = {Proc. of KR},
  year         = {2025}
}

@inproceedings{DBLP:conf/kr/MantenoglouKA23,
  author       = {Periklis Mantenoglou and
                  Dimitrios Kelesis and
                  Alexander Artikis},
  title        = {Complex Event Recognition with {Allen} Relations},
  booktitle    = {Proc. of KR},
  pages        = {502--511},
  year         = {2023}
}

@article{DBLP:journals/ai/Shahar97,
  author       = {Yuval Shahar},
  title        = {A Framework for Knowledge-Based Temporal Abstraction},
  journal      = {Artif. Intell.},
  volume       = {90},
  number       = {1-2},
  pages        = {79--133},
  year         = {1997}
}

@article{DBLP:journals/artmed/ShaharM96,
  author       = {Yuval Shahar and
                  Mark A. Musen},
  title        = {Knowledge-based temporal abstraction in clinical domains},
  journal      = {Artif. Intell. Medicine},
  volume       = {8},
  number       = {3},
  pages        = {267--298},
  year         = {1996}
}

@book{DBLP:books/aw/AbiteboulHV95,
  author       = {Serge Abiteboul and
                  Richard Hull and
                  Victor Vianu},
  title        = {Foundations of Databases},
  publisher    = {Addison-Wesley},
  year         = {1995},
  url          = {http://webdam.inria.fr/Alice/},
  isbn         = {0-201-53771-0},
  bibsource    = {dblp computer science bibliography, https://dblp.org}
}

@book{reiter-book,
    author = {Reiter, Raymond},
    title = {Knowledge in Action: Logical Foundations for Specifying and Implementing Dynamical Systems},
    publisher = {MIT Press},
    year = {2001}}

@book{DBLP:books/daglib/0095085,
  author       = {Murray Shanahan},
  title        = {Solving the {Frame} {Problem} - A Mathematical Investigation of the Common
                  Sense Law of Inertia},
  publisher    = {{MIT} Press},
  year         = {1997}
}

@ARTICLE{Augusto2005-ke,
  title     = "Temporal Reasoning for Decision Support in Medicine",
  author    = "Augusto, Juan Carlos",
  journal   = "Artif. Intell. Med.",
  publisher = "Elsevier BV",
  volume    =  33,
  number    =  1,
  pages     = "1--24",
  month     =  jan,
  year      =  2005,
  language  = "en"
}

@ARTICLE{Zhou2007-fw,
  title     = "Temporal Reasoning with Medical Data -- A Review with Emphasis on
               Medical Natural Language Processing",
  author    = "Zhou, Li and Hripcsak, George",
  journal   = "J. Biomed. Inform.",
  publisher = "Elsevier BV",
  volume    =  40,
  number    =  2,
  pages     = "183--202",
  month     =  apr,
  year      =  2007,
  copyright = "https://www.elsevier.com/open-access/userlicense/1.0/",
  language  = "en"
}

@inproceedings{DBLP:conf/ijcai/WalegaGKK20,
  author       = {Przemyslaw Andrzej Walega and
                  Bernardo Cuenca Grau and
                  Mark Kaminski and
                  Egor V. Kostylev},
  title        = {Tractable Fragments of Datalog with Metric Temporal Operators},
  booktitle    = {Proc. of {IJCAI}},
  year         = {2020}
}

@inproceedings{DBLP:conf/aaai/WangHWG22,
  author       = {Dingmin Wang and
                  Pan Hu and
                  Przemyslaw Andrzej Walega and
                  Bernardo Cuenca Grau},
  title        = {{MeTeoR}: Practical Reasoning in Datalog with Metric Temporal Operators},
  booktitle    = {Proc. of {AAAI}},
  year         = {2022}
}

@article{DBLP:journals/jair/BrandtKRXZ18,
  author       = {Sebastian Brandt and
                  Elem G{\"{u}}zel Kalayci and
                  Vladislav Ryzhikov and
                  Guohui Xiao and
                  Michael Zakharyaschev},
  title        = {Querying Log Data with Metric Temporal Logic},
  journal      = {J. Artif. Intell. Res.},
  volume       = {62},
  pages        = {829--877},
  year         = {2018}
}

@article{DBLP:journals/ngc/KowalskiS86,
  author       = {Robert A. Kowalski and
                  Marek J. Sergot},
  title        = {A Logic-based Calculus of Events},
  journal      = {New Gener. Comput.},
  volume       = {4},
  number       = {1},
  pages        = {67--95},
  year         = {1986},
  url          = {https://doi.org/10.1007/BF03037383},
  doi          = {10.1007/BF03037383},
  bibsource    = {dblp computer science bibliography, https://dblp.org}}

@article{DBLP:journals/tplp/MannaRT13,
  author       = {Marco Manna and
                  Francesco Ricca and
                  Giorgio Terracina},
  title        = {Consistent Query Answering via {ASP} From Different Perspectives:
                  Theory and Practice},
  journal      = {Theory Pract. Log. Program.},
  volume       = {13},
  number       = {2},
  pages        = {227--252},
  year         = {2013}
}

@article{DBLP:journals/tods/EiterFGL08,
  author       = {Thomas Eiter and
                  Michael Fink and
                  Gianluigi Greco and
                  Domenico Lembo},
  title        = {Repair Localization for Query Answering from Inconsistent Databases},
  journal      = {{ACM} Trans. Database Syst.},
  volume       = {33},
  number       = {2},
  pages        = {10:1--10:51},
  year         = {2008}
}

@article{DBLP:journals/tplp/GebserKKS19,
  author       = {Martin Gebser and
                  Roland Kaminski and
                  Benjamin Kaufmann and
                  Torsten Schaub},
  title        = {Multi-shot {ASP} Solving with clingo},
  journal      = {Theory Pract. Log. Program.},
  volume       = {19},
  number       = {1},
  pages        = {27--82},
  year         = {2019},
  url          = {https://doi.org/10.1017/S1471068418000054},
  doi          = {10.1017/S1471068418000054},
  bibsource    = {dblp computer science bibliography, https://dblp.org}
}

@book{asp-in-prac-2012,
  author    = {Martin Gebser and
               Roland Kaminski and
               Benjamin Kaufmann and
               Torsten Schaub},
  title     = {Answer Set Solving in Practice},
  publisher = {Morgan {\&} Claypool Publishers},
  year      = {2012},
  url       = {https://doi.org/10.2200/S00457ED1V01Y201211AIM019},
  doi       = {10.2200/S00457ED1V01Y201211AIM019},
  bibsource = {dblp computer science bibliography, https://dblp.org}
}

@article{DBLP:journals/cacm/BrewkaET11,
   author       = {Gerhard Brewka and
                  Thomas Eiter and
                  Miroslaw Truszczynski},
  title        = {Answer Set Programming at a Glance},
  journal      = {Commun. {ACM}},
  volume       = {54},
  number       = {12},
  pages        = {92--103},
  year         = {2011},
  url          = {https://doi.org/10.1145/2043174.2043195},
  doi          = {10.1145/2043174.2043195},
  bibsource    = {dblp computer science bibliography, https://dblp.org}
}

@book{DBLP:books/sp/Lifschitz19,
  author       = {Vladimir Lifschitz},
  title        = {Answer Set Programming},
  publisher    = {Springer},
  year         = {2019},
  url          = {https://doi.org/10.1007/978-3-030-24658-7},
  doi          = {10.1007/978-3-030-24658-7},
  bibsource    = {dblp computer science bibliography, https://dblp.org}
}

@book{ber,
  author       = {Leopoldo E. Bertossi},
  title        = {Database Repairing and Consistent Query Answering},
  series       = {Synthesis Lectures on Data Management},
  publisher    = {Morgan {\&} Claypool Publishers},
  year         = {2011}
}

@inproceedings{biebou,
  author       = {Meghyn Bienvenu and
                  Camille Bourgaux},
  title        = {Inconsistency-Tolerant Querying of Description Logic Knowledge Bases},
  booktitle    = {Reasoning Web Tutorial Lectures},
  	volume = {LNCS 4126},
  year         = {2016}
}

@inproceedings{biebougoa,
  author       = {Meghyn Bienvenu and
                  Camille Bourgaux and
                  Fran{\c{c}}ois Goasdou{\'{e}}},
  title        = {Querying Inconsistent Description Logic Knowledge Bases under Preferred
                  Repair Semantics},
  booktitle    = {Proceedings of {AAAI}},
  year         = {2014}
}

@article{hripcsak_Nextgeneration_2013,
	title = {Next-generation Phenotyping of Electronic Health Records},
	volume = {20},
	issn = {1067-5027},
	url = {https://www.ncbi.nlm.nih.gov/pmc/articles/PMC3555337/},
	doi = {10.1136/amiajnl-2012-001145},
	number = {1},
	urldate = {2021-12-10},
	journal = {J Am Med Inform Assoc},
	author = {Hripcsak, George and Albers, David J},
	year = {2013},
	pmid = {22955496},
	pmcid = {PMC3555337},
	pages = {117--121},
}

@article{li_Time_2020,
	title = {Time Event Ontology ({TEO}): To Support Semantic Representation and Reasoning of Complex Temporal Relations of Clinical Events},
	volume = {27},
	issn = {1067-5027},
	shorttitle = {Time event ontology ({TEO})},
	url = {https://www.ncbi.nlm.nih.gov/pmc/articles/PMC7647306/},
	doi = {10.1093/jamia/ocaa058},
	number = {7},
	urldate = {2023-05-03},
	journal = {J Am Med Inform Assoc},
	author = {Li, Fang and Du, Jingcheng and He, Yongqun and Song, Hsing-Yi and Madkour, Mohcine and Rao, Guozheng and Xiang, Yang and Luo, Yi and Chen, Henry W and Liu, Sijia and Wang, Liwei and Liu, Hongfang and Xu, Hua and Tao, Cui},
	month = jul,
	year = {2020},
	pmid = {32626903},
	pmcid = {PMC7647306},
	pages = {1046--1056},
}

@article{lyu_Temporal_2022,
	title = {Temporal {Events} {Detector} for {Pregnancy} {Care} ({TED}-{PC}): {A} Rule-Based Algorithm to Infer Gestational Age and Delivery Date From Electronic Health Records of Pregnant Women With and Without {COVID}-19},
	volume = {17},
	issn = {1932-6203},
	shorttitle = {Temporal {Events} {Detector} for {Pregnancy} {Care} ({TED}-{PC})},
	url = {https://www.ncbi.nlm.nih.gov/pmc/articles/PMC9621451/},
	doi = {10.1371/journal.pone.0276923},
	number = {10},
	urldate = {2023-05-03},
	journal = {PLoS One},
	author = {Lyu, Tianchu and Liang, Chen and Liu, Jihong and Campbell, Berry and Hung, Peiyin and Shih, Yi-Wen and Ghumman, Nadia and Li, Xiaoming},
	month = oct,
	year = {2022},
	pmid = {36315520},
	pmcid = {PMC9621451},
	pages = {e0276923},
}

@article{dwyer_Reasoning_2023,
	title = {Reasoning over Health Records With {Vadalog}: A Rule-Based Approach to Patient Pathways},
	volume = {3485},
	url = {https://ceur-ws.org/Vol-3485/paper9111.pdf},
	language = {en},
	journal = {Proc.  of Int. Rule Challenge @  RuleML+RR},
	author = {Dwyer, Owen P and Baldazzi, Teodoro and Davies, Jim and Sallinger, Emanuel and Vlad, Adriano},
	year = {2023},
	pages = {15},
}

@article{tsai_Effects_2020,
	title = {Effects of Electronic Health Record Implementation and Barriers to Adoption and Use: {A} Scoping Review and Qualitative Analysis of the Content},
	volume = {10},
	issn = {2075-1729},
	shorttitle = {Effects of {Electronic} {Health} {Record} {Implementation} and {Barriers} to {Adoption} and {Use}},
	url = {https://www.ncbi.nlm.nih.gov/pmc/articles/PMC7761950/},
	doi = {10.3390/life10120327},
	number = {12},
	urldate = {2023-11-18},
	journal = {Life (Basel)},
	author = {Tsai, Chen Hsi and Eghdam, Aboozar and Davoody, Nadia and Wright, Graham and Flowerday, Stephen and Koch, Sabine},
	month = dec,
	year = {2020},
	pmid = {33291615},
	pmcid = {PMC7761950},
	pages = {327},
}

@article{rance_Integrating_2016,
	title = {Integrating Heterogeneous Biomedical Data for Cancer Research: the {CARPEM} infrastructure},
	volume = {7},
	issn = {1869-0327},
	shorttitle = {Integrating {Heterogeneous} {Biomedical} {Data} for {Cancer} {Research}},
	url = {https://www.ncbi.nlm.nih.gov/pmc/articles/PMC4941838/},
	doi = {10.4338/ACI-2015-09-RA-0125},
	number = {2},
	urldate = {2023-11-19},
	journal = {Appl Clin Inform},
	author = {Rance, Bastien and Canuel, Vincent and Countouris, Hector and Laurent-Puig, Pierre and Burgun, Anita},
	month = may,
	year = {2016},
	pmid = {27437039},
	pmcid = {PMC4941838},
	pages = {260--274},
}

@inproceedings{DBLP:conf/lpnmr/CabalarKMS19,
  author       = {Pedro Cabalar and
                  Roland Kaminski and
                  Philip Morkisch and
                  Torsten Schaub},
  title        = {telingo = {ASP} + Time},
  booktitle    = {Proc. of LPNMR},
  pages        = {256--269},
  year         = {2019},
  url          = {https://doi.org/10.1007/978-3-030-20528-7\_19},
  doi          = {10.1007/978-3-030-20528-7\_19},
  bibsource    = {dblp computer science bibliography, https://dblp.org}
}

@incollection{cabalar_Temporal_2022,
	title = {Temporal {ASP}: {From} Logical Foundations to Practical Use with telingo},
	volume = {LNCS 13100},
	shorttitle = {Temporal {ASP}},
	url = {https://link.springer.com/10.1007/978-3-030-95481-9_5},
	language = {en},
	urldate = {2024-02-22},
	booktitle = {Reasoning {Web} Tutorial Lectures},
	author = {Cabalar, Pedro},
	year = {2022},
	pages = {94--114},
}

@article{DBLP:journals/csur/AlevizosSAP17,
  author       = {Elias Alevizos and
                  Anastasios Skarlatidis and
                  Alexander Artikis and
                  Georgios Paliouras},
  title        = {Probabilistic Complex Event Recognition: {A} Survey},
  journal      = {{ACM} Comput. Surv.},
  volume       = {50},
  number       = {5},
  pages        = {71:1--71:31},
  year         = {2017},
  url          = {https://doi.org/10.1145/3117809},
  doi          = {10.1145/3117809},
  bibsource    = {dblp computer science bibliography, https://dblp.org}
}

@article{artikis_Event_2015,
	title = {An Event Calculus for Event Recognition},
	volume = {27},
	issn = {1041-4347},
	url = {http://ieeexplore.ieee.org/document/6895142/},
	doi = {10.1109/TKDE.2014.2356476},
	language = {en},
	number = {4},
	urldate = {2024-03-18},
	journal = {IEEE Trans. Knowl. Data Eng.},
	author = {Artikis, Alexander and Sergot, Marek and Paliouras, Georgios},
	month = apr,
	year = {2015},
	pages = {895--908},
}

@article{falcionelli_Indexing_2019,
	title = {Indexing the Event Calculus: {Towards} practical Human-Readable Personal Health Systems},
	volume = {96},
	issn = {09333657},
	shorttitle = {Indexing the {Event} {Calculus}},
	url = {https://linkinghub.elsevier.com/retrieve/pii/S0933365717305948},
	doi = {10.1016/j.artmed.2018.10.003},
	language = {en},
	urldate = {2024-04-04},
	journal = {Artif. Intell. Med.},
	author = {Falcionelli, Nicola and Sernani, Paolo and Brugués, Albert and Mekuria, Dagmawi Neway and Calvaresi, Davide and Schumacher, Michael and Dragoni, Aldo Franco and Bromuri, Stefano},
	month = may,
	year = {2019},
	pages = {154--166},
}

@article{allen_Maintaining_1983,
	title = {Maintaining Knowledge about Temporal Intervals},
	  journal      = {Commun. {ACM}},
	volume = {26},
	language = {en},
	number = {11},
	author = {Allen, James F},
	year = {1983},
}

@article{awuklu_Ontologydriven_2025a,
	title = {Ontology-Driven Identification of Inconsistencies in Clinical Data: {A} Case Study in Lung Cancer Phenotyping},
	volume = {165},
	issn = {1532-0464},
	shorttitle = {Ontology-Driven Identification of Inconsistencies in Clinical Data},
	url = {https://www.sciencedirect.com/science/article/pii/S1532046425000371},
	doi = {10.1016/j.jbi.2025.104808},
	urldate = {2025-04-28},
	journal = {J. Biomed. Inform.},
	author = {Awuklu, Yvon K. and Mougin, Fleur and Griffier, Romain and Bienvenu, Meghyn and Jouhet, Vianney},
	month = may,
	year = {2025},
	pages = {104808},
}

@incollection{DBLP:books/mk/minker88/AptBW88,
  author       = {Krzysztof R. Apt and
                  Howard A. Blair and
                  Adrian Walker},
  editor       = {Jack Minker},
  title        = {Towards a Theory of Declarative Knowledge},
  booktitle    = {Foundations of Deductive Databases and Logic Programming},
  pages        = {89--148},
  publisher    = {Morgan Kaufmann},
  year         = {1988},
  url          = {https://doi.org/10.1016/b978-0-934613-40-8.50006-3},
  doi          = {10.1016/B978-0-934613-40-8.50006-3},
  bibsource    = {dblp computer science bibliography, https://dblp.org}
}

@article{DBLP:journals/csur/DantsinEGV01,
  author       = {Evgeny Dantsin and
                  Thomas Eiter and
                  Georg Gottlob and
                  Andrei Voronkov},
  title        = {Complexity and Expressive Power of Logic Programming},
  journal      = {{ACM} Comput. Surv.},
  volume       = {33},
  number       = {3},
  pages        = {374--425},
  year         = {2001},
  url          = {https://doi.org/10.1145/502807.502810},
  doi          = {10.1145/502807.502810},
  bibsource    = {dblp computer science bibliography, https://dblp.org}
}

@article{uzuner_Evaluating_2007,
	title = {Evaluating the {State}-of-the-{Art} in {Automatic} {De}-identification},
	volume = {14},
	issn = {1067-5027},
	url = {https://www.ncbi.nlm.nih.gov/pmc/articles/PMC1975792/},
	doi = {10.1197/jamia.M2444},
	number = {5},
	urldate = {2025-07-29},
	journal = {J Am Med Inform Assoc},
	author = {Uzuner, Ozlem and Luo, Yuan and Szolovits, Peter},
	year = {2007},
	pmid = {17600094},
	pmcid = {PMC1975792},
	pages = {550--563},
}

@article{hripcsak_Agreement_2005,
	title = {Agreement, the {F}-{Measure}, and {Reliability} in {Information} {Retrieval}},
	volume = {12},
	issn = {1067-5027},
	url = {https://www.ncbi.nlm.nih.gov/pmc/articles/PMC1090460/},
	doi = {10.1197/jamia.M1733},
	number = {3},
	urldate = {2025-07-29},
	journal = {J Am Med Inform Assoc},
	author = {Hripcsak, George and Rothschild, Adam S.},
	year = {2005},
	pmid = {15684123},
	pmcid = {PMC1090460},
	pages = {296--298},
}

@article{artstein_InterCoder_2008,
	title = {Inter-{Coder} {Agreement} for {Computational} {Linguistics}},
	volume = {34},
	issn = {0891-2017},
	url = {https://doi.org/10.1162/coli.07-034-R2},
	doi = {10.1162/coli.07-034-R2},
	number = {4},
	urldate = {2025-07-29},
	journal = {Computational Linguistics},
	author = {Artstein, Ron and Poesio, Massimo},
	month = dec,
	year = {2008},
	pages = {555--596},
}

\clearpage
\newpage

\appendix

\section{Proofs for Section \ref{complexity}}

\noindent\textbf{Theorem~\ref{thm:ptimegen}}. 
The sets $\infevents$ and $ \mathsf{ME}(\dataset,\mathcal{S}, \eventspec)$ can be computed in \textsc{PTime} in data complexity. 
\begin{proof}
Note that the set $\simplerules(\dataset)$ is clearly computable in \textsc{PTime} data complexity as it 
essentially corresponds to evaluation of first-order queries over a database / first-order structure, a task 
which is well known to be computable in (sub)polynomial time\footnote{We direct interested readers to Chapter 17.1 of \cite{DBLP:books/aw/AbiteboulHV95} for more information on the complexity of first-order query evaluation and a proof of membership in $\mathsf{AC}_0$. } in data complexity (more precisely, in $\mathsf{AC}_0 \subseteq \textsc{LogSpace}$). 
We assume of course that any considered 
numeric functions are \textsc{PTime}-computable.  
Likewise, a routine argument can be used to show 
that $ \mathsf{ME}(\dataset,\mathcal{S}, \eventspec)$ is \textsc{PTime}-computable once we have already computed $\infevents$, as this essentially 
corresponds to evaluation of stratified Datalog programs\footnote{The complexity of reasoning with different extensions of Datalog with negation, including stratified negation, can be found in \cite{DBLP:journals/csur/DantsinEGV01}.}. We shall therefore concentrate on explaining how to compute $\infevents$ from the sets $T_\exists^\ell(R(\vect{d}))$ and $T_\stop^\ell(R(\vect{d}))$ and the window provided by $\windowpred(R(\vect{d} ),w) \in \simplerules(\dataset)$. 

First take some $R \in \nspreds$ such that $T_\exists^\ell(R(\vect{d})) \neq \emptyset$ for some $\ell$, and let $w$ be such that 
$\windowpred(R(\vect{d} ),w) \in \simplerules(\dataset)$. 
The general idea is that when computing the facts for confidence level $\ell$, we can start from the timepoints in $T_\exists^\ell(R(\vect{d}))$, 
then iteratively expand these initial intervals until we either get blocked by a termination timepoint in $T_\stop^\ell(R(\vect{d}))$
or cannot find a further nearby timepoint in $T_\exists^\ell(R(\vect{d}))$. 
To make this more formal, suppose we are considering confidence level $\ell$ and have already treated all confidence 
levels $< \ell$. We initialize the set $I_\ell$ with the intervals $[t,t]$ such that $t \in T_\exists^\ell(R(\vect{d})) \setminus T_\stop^\ell(R(\vect{d}))$, and $I_\ell^\times$ with those $[t,t]$ such that $t \in T_\exists^\ell(R(\vect{d})) \cap T_\stop^\ell(R(\vect{d}))$. 
Then until we reach a fixpoint, we pick an interval $[t_1, t_2]$ in $I_\ell \cup I_\ell^\times$ 
that has not yet been considered and do the following: 
\begin{itemize}
\item If there exists $t_1^\sharp \in T_\exists^\ell(R(\vect{d}))$ with $t_1 - w \leq t_1^\sharp < t_1$
such that there is no $t_e \in T_\stop^\ell(R(\vect{d}))$
with $t_1^\sharp \leq t_e < t_1$, then add $[t_1^\sharp, t_2]$ to $I_\ell$ (or to $I_\ell^\times$ if $[t_1,t_2] \in I_\ell^\times$)
\item If $[t_1, t_2] \in I_\ell$ and there exists $t_2^\sharp \in T_\exists^\ell(R(\vect{d}))$ with $t_2^\sharp  - t_2 \leq w$
such that there is no $t_e\in T_\stop^\ell(R(\vect{d}))$  with $t_2 \leq t_e < t_2^\sharp$, 
then add $[t_1, t_2^\sharp]$ to $I_\ell$
\item If $[t_1, t_2] \in I_\ell$ and there exists $t_e \in T_\stop^\ell(R(\vect{d}))$ with $t_e  - t_2 \leq w$
and there is no other $t'_e\in T_\stop^\ell(R(\vect{d}))$  with $t_2 \leq t_e' < t_e$, 
then add $[t_1, t_e]$ to $I_\ell^\times$
\end{itemize}
Observe that intervals in $I_\ell$ may be expanded in either direction, while those in $I_\ell^\times$ may only be expanded w.r.t.\ their starting timepoint, since we know by construction that the second timepoint satisfies a termination condition. 
Also note that since there are only quadratically many intervals that may be produced, and the preceding conditions can be checked in polynomial time, the whole process will terminate in polynomial time. 
By comparing the above conditions with the items in Definition \ref{nspred-def}, 
we can see that we are sure to produce all of the intervals $[t_1,t_2]$ such that $R(\vect{d}, [t_1,t_2])$ is inferred with confidence $\ell$. 
Thus, it only remains to compare the intervals in $I_\ell$ to prune the non-maximal ones (and to eliminate those which appear in a previous confidence level), 
which can be clearly done in \textsc{PTime}. 
The resulting set of intervals tells us precisely which facts $R(\vect{d}, [t_1, t_2], \ell)$ are present in $\infevents$. 

Now suppose we have $R \in \pspreds$ such that $T_\exists^\ell(R(\vect{d})) \neq \emptyset$ for some $\ell$.
The argument is simpler in this case. Intuitively, given a timepoint $t \in T_\exists^\ell(R(\vect{d}))$, we make a linear pass on the timepoints starting from $t$ and look for the first occurrence of a termination condition. Formally, we let $I_\ell$ contain all intervals $[t_1,t_2]$ such that one of the following two conditions holds:
\begin{itemize}
\item $t_1 \in T_\exists^\ell(R(\vect{d}))$, $t_2 \in T_\stop^\ell(R(\vect{d}))$, and there is no $t_2^\sharp\in T_\stop^\ell(R(\vect{d}))$ such that $t_1 \leq t_2^\sharp < t_2$
\item $t_1 \in T_\exists^\ell(R(\vect{d}))$, $t_2 = \now$, and there is no $t_2^\sharp\in T_\stop^\ell(R(\vect{d}))$ such that $t_1 \leq t_2^\sharp $
\end{itemize}
Similarly to the non-persistent case, we can argue that this set can be constructed in polynomial time and that it is sure to include the required intervals. It then suffices to remove non-maximal or redundant intervals, again a polynomial-time operation. 
\end{proof}

\medskip

We give separate proofs for the upper and lower bounds stated in Theorem~\ref{thm:conpgen}. 
First, we establish \textsc{coNP} membership in data complexity of the problems of recognizing consistent and preferred timelines in the general case.

\begin{proposition}\label{conp-ub}
It is in \textsc{coNP} in data complexity to decide, given a TES $\eventspec$, dataset $\dataset$, and set of facts $\mathcal{S}$,
whether $\mathcal{S}$ is a consistent (or preferred) timeline for $\eventspec, \dataset$.
\end{proposition} 
\begin{proof}
We start with the upper bound for consistent timelines. Consider the following guess-and-check procedure, whose input is 
a TES $\eventspec$, dataset $\dataset$, and set of facts~$\mathcal{S}$ (which may contain both simple event and meta-event facts):
\begin{enumerate}
\item Compute $\infevents$ and $\mathcal{S}_{\mathsf{SE}} = \mathcal{S} \cap \infevents$. 
\item Guess a subset $\mathcal{S}'_{\mathsf{SE}}$ of $\infevents$.
\item Check if $\mathcal{S}$ and $\mathcal{S}'$ are $\eventspec$-consistent. 
\item Return `yes' if one of the following conditions holds (else return `no'): 
\begin{enumerate}
\item $\mathcal{S} \neq \mathcal{S}_{\mathsf{SE}} \cup \mathsf{ME}(\dataset,\mathcal{S}_{\mathsf{SE}}, \eventspec)$
\item $\mathcal{S}_{\mathsf{SE}}$ is \emph{not} $\eventspec$-consistent, or
\item  $\mathcal{S}'_{\mathsf{SE}}$ is $\eventspec$-consistent and $\mathcal{S}_{\mathsf{SE}} \subsetneq \mathcal{S}'_{\mathsf{SE}}$
\end{enumerate}
\end{enumerate}
We claim that some execution of this non-deterministic procedure returns `yes' iff $\mathcal{S}$ is \emph{not} a consistent timeline. 
Indeed, if an execution returns `yes', then one of (a), (b), or (c) is satisfied. If condition (a) is satisfied, then $\mathcal{S}$
does not have the required form to be a consistent timeline. If condition (b) or (c) is satisfied, then $\mathcal{S}_{\mathsf{SE}}$
is either $\eventspec$-inconsistent, or there is a larger subset $\mathcal{S}'_{\mathsf{SE}}$ of $\infevents$ that is $\eventspec$-consistent. 
In either case, 
we can infer that $\mathcal{S}_{\mathsf{SE}} \not \in \reps(\infevents, \eventspec)$, and hence that $\mathcal{S}$ is not 
a consistent timeline. Conversely, suppose that $\mathcal{S}$ is not a consistent timeline. The first possibility is that 
$\mathcal{S}$ is not equal to $\mathcal{R} \cup \mathsf{ME}(\dataset,\mathcal{R}, \eventspec)$ for some subset $\mathcal{R} \subseteq \infevents$. 
This can occur either because $\mathcal{S}$ contains a simple event fact not included in $\infevents$ or 
because the meta-events in $\mathcal{S}$ do not match $\mathsf{ME}(\dataset,\mathcal{S}_{\mathsf{SE}},\eventspec)$. In both cases, we obtain $\mathcal{S} \neq \mathcal{S}_{\mathsf{SE}} \cup \mathsf{ME}(\dataset,\mathcal{S}_{\mathsf{SE}}, \eventspec)$,
so the procedure will return `yes' (irrespective of the guessed set). So let us consider the other case, in which 
$\mathcal{S} = \mathcal{S}_{\mathsf{SE}} \cup \mathsf{ME}(\dataset,\mathcal{S}_{\mathsf{SE}}, \eventspec)$. 
As $\mathcal{S}$ is not a consistent timeline, it must be the case that $\mathcal{S}_{\mathsf{SE}} \not \in \reps(\infevents, \eventspec)$. 
This can be for two reasons: either $\mathcal{S}_{\mathsf{SE}}$ is not $\eventspec$-consistent,
or we can find a larger $\eventspec$-consistent subset of $\infevents$. In the former case, we will return `yes' due to condition (b),
and in the latter case, we can consider an execution in which we 
guess this larger $\eventspec$-consistent of $\infevents$ in Step 2 and will return `yes' due to condition (c). 

It is easy to see that this procedure runs in non-deterministic polynomial time (w.r.t.\ data complexity). 
Indeed, we have shown in Theorem \ref{thm:ptimegen} that the sets $\infevents$ and $\mathsf{ME}(\dataset,\mathcal{S}_{\mathsf{SE}}, \eventspec)$
are computable in \textsc{PTime} data complexity. In particular, this means that the set $\infevents$ has polynomial size in data complexity,
and so the set $\mathcal{S}'_{\mathsf{SE}}$ guessed in Step 2 is of polynomial size. The consistency checks in Step 3 can also be 
performed in \textsc{PTime} data complexity, as it suffices to consider each of the constraints $\bot \gets C$ in $\eventspec$ and check whether $C$ (viewed as a first-order sentence)
evaluates to true w.r.t.\ $\dataset \cup \mathcal{S}^-_{\mathsf{SE}} \cup \mathsf{ME}^-(\dataset,\mathcal{S}_{\mathsf{SE}}, \eventspec)$. 
We thus have an \textsc{NP} procedure for deciding whether a given set of facts is \emph{not} a consistent timeline, 
which immediately yields the desired \textsc{coNP} upper bound for the original task
of recognizing consistent timelines. 

\smallskip

We can straightforwardly adapt the preceding procedure to show the upper bound for preferred timelines. 
Indeed, it suffices to replace condition (c) with the condition (c'), given as follows:  
\begin{itemize}
\item[]  $\mathcal{S}'_{\mathsf{SE}}$ is $\eventspec$-consistent 
and there exists some confidence level $k$ such that
(i)~$(\mathcal{S}'_{\mathsf{SE}})_\ell = (\mathcal{S}_{\mathsf{SE}})_\ell$ for every $1 \leq \ell < k$, 
and (ii)~
$(\mathcal{S}_{\mathsf{SE}})_k  \subsetneq (\mathcal{S}'_{\mathsf{SE}})_k$. 
\end{itemize}
Indeed, it follows from Definitions \ref{prefrepdef} and \ref{timelinedef} that if $\mathcal{S}$ is a consistent timeline but not a preferred timeline,
then there exists a $\Sigma$-consistent $\mathcal{U} \subseteq \mathsf{SE}(\dataset,\Sigma)$ and $k \geq 1$ such that
(i)~$\mathcal{U}_\ell = (\mathcal{S}_{\mathsf{SE}})_\ell$ for every $1 \leq \ell < k$, and (ii)~
$(\mathcal{S}_{\mathsf{SE}})_k \subsetneq \mathcal{U}_k$. 
Thus, by replacing (c) with (c'), we obtain a procedure for checking whether 
a set is not a preferred timeline. This establishes \textsc{coNP} membership of the complementary problem 
of recognizing preferred timelines.  
\end{proof}

To establish the \textsc{coNP}-hardness result from Theorem~\ref{thm:conpgen}, the basic idea is to start from a propositional CNF formula and generate simple events that contains facts encoding all possible truth values for the formula's variables, then use constraints involving both positive and negative event atoms to enforce that consistent timelines pick a set of event atoms that define a propositional valuation. An additional constraint, again involving negated atoms, is used to ensure that there is no unsatisfied clause. 

\begin{proposition}\label{conp-lb}
It is \textsc{coNP}-hard in data complexity to decide, given a TES $\eventspec$, dataset $\dataset$, and set of facts $\mathcal{S}$,
whether $\mathcal{S}$ is a consistent (or preferred) timeline for $\eventspec, \dataset$.
\end{proposition}
\begin{proof}
We first note that it is sufficient to prove a \textsc{coNP}
lower bound for consistent timelines, since preferred timelines coincide with consistent timelines
when there is a single confidence level. 

We reduce the well-known \textsc{NP}-complete problem, 
3SAT, of testing the satisfiability of propositional 3CNF formulas
to the problem of testing whether a set of facts is \emph{not} a consistent timeline. 
Consider a propositional 3CNF $\varphi = \lambda_1 \wedge \ldots \wedge \lambda_m$ over propositional variables $v_1, \ldots, v_k$,
where each $\lambda_i= l_{i,1} \vee l_{i,2} \vee l_{i,3} $ is a clause consisting of 3 literals $l_{i,1}, l_{i,2}, l_{i,3}$.
It will be convenient to associate with each clause $\lambda_i$ a corresponding vector $(v_{i,1}, b_{i,1}, v_{i,2}, b_{i,2}, v_{i,3}, b_{i,3})$
where $v_{i,j}$ is the variable in literal $l_{i,j}$ and $b_{i,j}=1$ (resp.\ $b_{i,j}=0$) if $l_{i,j} = v_{i,j}$ (resp.\ $l_{i,j} = \neg v_{i,j}$). 
We will use the following dataset $\dataset_\varphi$ to encode $\varphi$:
\begin{align*}
\dataset_\varphi = & \{\mathsf{Var}(v_i, 0) \mid 1 \leq  i \leq k\}\\
& \cup \{\mathsf{Clause}(v_{i,1}, b_{i,1}, v_{i,2}, b_{i,2}, v_{i,3}, b_{i,3}) \mid 1 \leq i \leq m  \}
\end{align*}
which uses $v_1, \ldots, v_k, 0, 1$ as constants. There is a single observation predicate $\mathsf{Var}$ (of arity 1), whose facts designate the available propositional variables and all use the same timepoint (0). Each clause $\lambda_i$ is encoded using an atemporal fact with the 6-ary atemporal predicate $\mathsf{Clause}$ and the list of arguments corresponding to the vector representation of $\lambda_i$, which gives the three pairs (variable, truth value) that make the clause true. 

We will use the TES $\eventspec= (\simplerules, \emptyset, \tconstrain, \domconstrain)$, 
which does not contain any meta-events and uses two simple events $\mathsf{Q}$ (arity 0) and $\mathsf{Value}$ (arity 2).
We choose to model both as persistent simple events (though the argument can be adapted to use non-persistent events instead). 
The set $\simplerules$ will consist of the following three existence rules, all having confidence level 1: 
\begin{align*}
\existpred(\mathsf{Q}, t, 1) \gets & \mathsf{Var}(x,t)\\
\existpred(\mathsf{Value}(x,1), t, 1) \gets & \mathsf{Var}(x,t)\\
\existpred(\mathsf{Value}(x,0), t, 1) \gets & \mathsf{Var}(x,t)
\end{align*}
The set $ \domconstrain$ of domain constraints contains 
the following three constraints:
\begin{align*}
\gets\,\, & \mathsf{Value}(x,1,[t,t']) \wedge \mathsf{Value}(x,0,[t,t'])\\
\gets\,\, &  \mathsf{Var}(x,t) \wedge \mathsf{Q}([t,t']) \\
& \wedge \neg \mathsf{Value}(x,1,[t,t']) \wedge \neg \mathsf{Value}(x,0,[t,t'])      \\
\gets\,\, &  \mathsf{Clause}(x_1,y_1,x_2,y_2, x_3,y_3) \wedge \mathsf{Q}([t,t'])\\ 
&\wedge \neg \mathsf{Value}(x_1,y_1,[t,t']) \wedge \neg \mathsf{Value}(x_2,y_2,[t,t']) \\
& \wedge \neg \mathsf{Value}(x_3,y_3,[t,t']) 
\end{align*}
Note that these constraints satisfy the safety condition as all variables appearing in negated atoms are also present in a positive atom. 
Importantly, $\eventspec$ does not depend on the instance $\varphi$ 
(as required for a data complexity reduction).

It is easily verified that we obtain the following set of inferred simple events:
\begin{align*}
\mathsf{SE}(\dataset_\varphi, \eventspec)= & \{\mathsf{Q}([0,\now],1)\} \cup\\ 
& \{\mathsf{Value}(v_i,b, [0,\now],1) \mid b \in \{0,1\}, 1 \leq i \leq k \}
\end{align*}
To show the correctness of the reduction, 
we establish the following claim: \smallskip

\noindent\textbf{Claim}: $\{ \mathsf{Q}([0,\now], 1) \}$ is \emph{not} a consistent timeline of ($\eventspec,\dataset_\varphi$)  
iff $\varphi$ is satisfiable.\smallskip\\
$(\Rightarrow)$. First suppose that $\{ \mathsf{Q}([0,\now], 1) \}$ is \emph{not} a consistent timeline. 
It is clear that it is $\eventspec$-consistent and closed under meta-event rules (since $\metarules=\emptyset$).
It follows then that  $\{ \mathsf{Q}([0,\now], 1) \} \not \in \reps(\mathsf{SE}(\dataset_\varphi, \eventspec), \eventspec)$,
and thus that there exists a larger $\eventspec$-consistent set $\mathcal{S} \subseteq \mathsf{SE}(\dataset_\varphi, \eventspec)$
that contains $\mathsf{Q}([0,\now], 1)$. However, due to the second constraint, $\mathcal{S}$ must contain 
either $\mathsf{Value}(v_i,0, [0,\now],1)$ or $\mathsf{Value}(v_i,1, [0,\now],1)$ for each $1 \leq i \leq k$. The first constraint
ensures that $\mathcal{S}$ must contain precisely one of these two facts, for every $1 \leq i \leq k$.
We can thus define a valuation $\mu$ of the variables $v_1, \ldots, v_k$, by setting $\mu(v_i)=1$ if 
$\mathcal{S}$ contains $\mathsf{Value}(v_i,1, [0,\now],1)$
and $\mu(v_i)=0$ if 
$\mathcal{S}$ contains $\mathsf{Value}(v_i,0, [0,\now],1)$. 
Due to the third constraint and the encoding of clauses in $\dataset_\varphi$,
we know that $\mu$ must satisfy all of the clauses (since the third constraint 
is violated if each of the literals in the clause is not satisfied by the valuation
defined by the $\mathsf{Value}$ facts). It follows that $\varphi$ is satisfiable. \smallskip

\noindent $(\Leftarrow)$. Now suppose that $\varphi$ is satisfiable, and let $\mu$
be a satisfying valuation. Consider the set $\mathcal{S}$ defined as follows:
\begin{align*}
\mathcal{S}= 
& \{\mathsf{Q}([0,\now],1)\} \cup \{\mathsf{Value}(v_i,\mu(v_i), [0,\now],1) \mid 1 \leq i \leq k \}
\end{align*}
We can then verify that $\mathcal{S}$ is $\eventspec$-consistent. The first 
two constraints are satisfied since there is precisely one $\mathsf{Value}$ fact
per $v_i$, and the third constraint is satisfied since $\mu$ is a satisfying 
valuation. This shows that $\{\mathsf{Q}([0,\now],1)\}$ is not a repair 
and hence is not a consistent timeline. 
\end{proof}

The reduction used to show \textsc{coNP}-hardness employed a TES with negated event atoms. 
We show that this is necessary, as the recognition problems become tractable if we disallow negated event atoms in rules and constraints
(note that negation can still be applied to the atemporal and observation atoms). The key to obtaining tractability is to show that inconsistency is monotonic, which means that if a set of simple event facts is consistent and not a repair, then there must exist a single fact that can be added while retaining consistency, thereby witnessing that the original set is not a repair. This approach to testing whether a set is maximally consistent has been used in numerous KR settings, and in particular, in prior work on repairs of knowledge bases (see e.g.\ Lemma 1 of \cite{biebou}). \medskip

\noindent \textbf{Theorem \ref{thm:ptime-noneg}}. 
Given a TES $\eventspec$ without negated event atoms, a dataset $\dataset$, and a set of facts $\mathcal{S}$,
it can be decided in \textsc{PTime} whether $\mathcal{S}$ is a consistent (or preferred) timeline for $\eventspec, \dataset$.
\begin{proof}
The central property we use is monotonicity of inconsistency:  if a set of facts $\mathcal{S} \subseteq \infevents$ is $\eventspec$-inconsistent,
then every set  $\mathcal{S}' $ with $\mathcal{S} \subseteq  \mathcal{S}' \subseteq \infevents$ is also $\eventspec$-inconsistent. It is easy to see that this property holds 
when meta-event rules and constraints cannot use negated event atoms. Indeed, in the absence of negated atoms, 
$\mathcal{S} \subseteq  \mathcal{S}'$ implies $\mathsf{ME}(\dataset,\mathcal{S}, \eventspec) \subseteq \mathsf{ME}(\dataset,\mathcal{S}', \eventspec)$, 
and hence 
\begin{align*}
& \dataset \cup \mathcal{S}^- \cup \mathsf{ME}^-(\dataset,\mathcal{S}, \eventspec)
\subseteq  \dataset \cup (\mathcal{S')}^- \cup \mathsf{ME}^-(\dataset,\mathcal{S}', \eventspec)
\end{align*}
Since the constraints also cannot refer to negated event atoms, the latter inclusion
implies that if a constraint is violated by $\dataset \cup \mathcal{S}^- \cup \mathsf{ME}^-(\dataset,\mathcal{S}, \eventspec)$,
it will also be violated by $ \dataset \cup (\mathcal{S')}^- \cup \mathsf{ME}^-(\dataset,\mathcal{S}', \eventspec)$. 
 
 With this monotonicity property at hand, we can adopt a simple (and oft-used) approach
 to test whether a set $\mathcal{S}_{\mathsf{SE}} \subseteq \infevents$ is maximally consistent and hence a repair:
 \begin{itemize}
 \item test whether $\mathcal{S}_{\mathsf{SE}}$ is $\eventspec$-consistent
 \item for each $\sigma \in \infevents \setminus \mathcal{S}_{\mathsf{SE}}$, test whether $\mathcal{S}_{\mathsf{SE}} \cup \sigma$ is $\eventspec$-inconsistent
 \end{itemize}
If the first consistency check succeeds, and all of the candidate supersets in the second item are shown inconsistent,
then we can be sure that $\mathcal{S}_{\mathsf{SE}}$ is maximally consistent, hence a repair. 
It follows that to check whether $\mathcal{S}$ is a consistent timeline, we can use the following procedure: 
\begin{enumerate}
\item Compute $\infevents$ and $\mathcal{S}_{\mathsf{SE}} = \mathcal{S} \cap \infevents$. 
\item Check if $\mathcal{S} = \mathcal{S}_{\mathsf{SE}} \cup \mathsf{ME}(\dataset,\mathcal{S}_{\mathsf{SE}}, \eventspec)$. Return `no' if not. 
\item Check if $\mathcal{S}_{\mathsf{SE}}$ is $\eventspec$-consistent. Return `no' if not. 
\item For each $\sigma \in \infevents \setminus \mathcal{S}_{\mathsf{SE}}$, check if $\mathcal{S}_{\mathsf{SE}} \cup \{\sigma\}$ is $\eventspec$-consistent. 
If some set $\mathcal{S}_{\mathsf{SE}} \cup \{\sigma\}$ is $\eventspec$-consistent, return `no', else return `yes'. 
\end{enumerate}
The procedure clearly runs in \textsc{PTime} in data complexity and is correct due to the preceding
characterization of maximal consistent sets of $\infevents$. 

Let us now turn to preferred repairs. Using the monotonicity property
and Definition \ref{prefrepdef}, it is easy to see that 
if a set $\mathcal{S}_{\mathsf{SE}} \in \reps(\infevents, \Sigma)$ does not belong to 
$\prefreps(\infevents, \Sigma)$, then there exists some $\ell \geq 1$ 
and $\sigma_\ell \in \infevents_\ell \setminus \mathcal{S}_{\mathsf{SE}}$ 
such that $(\mathcal{S}_{\mathsf{SE}})_\ell \cup \{\sigma_\ell\}$ is $\eventspec$-consistent.
It follows that we can adapt the procedure for consistent timelines to be able to recognize 
preferred timelines by removing `else return yes' from Step 4 and adding the following Step 5, 
where $n$ is the maximal confidence level mentioned in $\infevents$:
\begin{enumerate}
\item[5.] For every $1 \leq \ell \leq n$ and for every $\sigma_\ell \in \infevents_\ell \setminus \mathcal{S}_{\mathsf{SE}}$, 
test whether $(\mathcal{S}_{\mathsf{SE}})_\ell \cup \{\sigma_\ell\}$ is $\eventspec$-consistent. Return `no' if some consistency check succeeds,
else return `yes'. 
\end{enumerate}
Note that this step remains polynomial-time computable since it involves only polynomially many 
consistency checks, and each consistency check can be done in \textsc{PTime} data complexity (cf.\ proof of Proposition \ref{conp-ub}). 
\end{proof}

For the cautious timeline, however, the absence of negated event atoms does not suffice to 
ensure tractability. \medskip

\noindent\textbf{Theorem \ref{hardness-cautious}}. 
It is \textsc{coNP}-hard to decide, given a TES $\eventspec$, dataset $\dataset$, and set of facts $\mathcal{S}$, whether $\mathcal{S}$ is the cautious timeline for $\eventspec, \dataset$. The lower bound holds even if we restrict to $\eventspec$ without negated event atoms. 
\begin{proof}
The general proof strategy is inspired by a reduction that was used to show \textsc{coNP}-hardness of 
the problem of testing whether a fact holds in every repair of a knowledge base formulated in the $\mathcal{EL}_\bot$ description logic
( \citeauthor{biebou}, 2016, proof of Theorem 29). We also reuse components of the proof of Proposition \ref{conp-lb}. 

We reduce 3SAT to the problem of deciding if a set is \emph{not} the cautious timeline. 
Consider a 3SAT instance $\varphi = \lambda_1 \wedge \ldots \wedge \lambda_m$ over propositional variables $v_1, \ldots, v_k$,
with $\lambda_i= l_{i,1} \vee l_{i,2} \vee l_{i,3} $. As in the proof of Proposition \ref{conp-lb}, we associate a vector 
$(v_{i,1}, b_{i,1}, v_{i,2}, b_{i,2}, v_{i,3}, b_{i,3})$ with each clause $\lambda_i$, and we define the dataset $\dataset_\varphi$
in almost the same way: 
\begin{align*}
\dataset_\varphi = & \{\mathsf{Var}(v_i, 0) \mid 1 \leq  i \leq k\} \cup\\
&  \{\mathsf{Clause}(c_i, v_{i,1}, b_1, v_{i,2}, b_2, v_{i,3}, b_3) \mid 1 \leq i \leq m  \} \cup \\
& \{\mathsf{Next}(c_i, c_{i+1}) \mid 1 \leq i < m \} \cup \{\mathsf{First}(c_1), \mathsf{Last}(c_m)\}
\end{align*}
Note that this resembles the dataset from the proof of~Proposition \ref{conp-lb}, except that we add an additional 
argument $c_i$ to the $\mathsf{Clause}$ facts and add further atemporal facts with predicates $\mathsf{First}, \mathsf{Next}, \mathsf{Last}$ 
to identify the first and last clauses and link subsequent clauses. 

We will use the TES $\eventspec= (\simplerules, \metarules, \tconstrain, \domconstrain)$, 
with the same set of simple event rules $\simplerules$ as used in the proof of~Proposition \ref{conp-lb}, recalled here for convenience: 
\begin{align*}
\existpred(\mathsf{Q}, t, 1) \gets & \mathsf{Var}(x,t)\\
\existpred(\mathsf{Value}(x,1), t, 1) \gets & \mathsf{Var}(x,t)\\
\existpred(\mathsf{Value}(x,0), t, 1) \gets & \mathsf{Var}(x,t)
\end{align*}
However, we will now introduce a meta-event predicate $\mathsf{Sat}$
defined using the following six rules in $\metarules$ (note that each of the two rules below 
is instantiated for $j=1,2,3$):
\begin{align*}
\mathsf{Sat}(z, [t,t'], 1) \gets &  \mathsf{Q}([t,t']) \wedge  \mathsf{First}(z) \qquad (1 \leq j \leq 3) \\
& \wedge \mathsf{Clause}(z, x_1,y_1,x_2,y_2, x_3,y_3) \\
& \wedge \mathsf{Value}(x_j,y_j,[t,t']) \\
\mathsf{Sat}(z', [t,t'], 1) \gets & \mathsf{Sat}(z, [t,t'], 1) \wedge \mathsf{Q}([t,t']) \\
& \wedge  \mathsf{Next}(z,z') \qquad (1 \leq j \leq 3) \\ 
& \wedge \mathsf{Clause}(z', x_1,y_1,x_2,y_2, x_3,y_3) \\
& \wedge \mathsf{Value}(x_j,y_j,[t,t'])
\end{align*}
The set of domain constraints $\domconstrain$ consists of:
\begin{align*}
\gets\,\, & \mathsf{Value}(x,1,[t,t']) \wedge \mathsf{Value}(x,0,[t,t'])\\
\gets\,\, &  \mathsf{Q}([t,t']) \wedge \mathsf{Last}(z) \wedge \mathsf{Sat}(z,[t,t'])
\end{align*}
Observe that no negated atoms appear in the TES and that we have the 
same set of inferred simple events as in the proof of~Proposition \ref{conp-lb}: 
\begin{align*}
\mathsf{SE}(\dataset_\varphi, \eventspec)= & \{\mathsf{Q}([0,\now],1)\} \cup\\ 
& \{\mathsf{Value}(v_i,b, [0,\now],1) \mid b \in \{0,1\}, 1 \leq i \leq k \}
\end{align*}
To show the correctness of the reduction, 
we establish the following claim: \smallskip

\noindent\textbf{Claim}: $\{ \mathsf{Q}([0,\now], 1) \}$ is \emph{not} the cautious timeline of ($\eventspec,\dataset_\varphi$)  
iff $\varphi$ is satisfiable.\smallskip\\
$(\Rightarrow)$. First suppose that $\{ \mathsf{Q}([0,\now], 1) \}$ is \emph{not} the cautious timeline. 
Note that this means that the cautious timeline must be equal to the empty set. To see why, first observe 
that for every fact 
$\mathsf{Value}(v_i,b, [0,\now],1) \in \mathsf{SE}(\dataset_\varphi, \eventspec)$, 
we can find a repair that contains $\mathsf{Value}(v_i,1-b, [0,\now],1)$ and hence must omit 
$\mathsf{Value}(v_i,b, [0,\now],1)$ to ensure consistency with the first constraint. It follows that 
no fact of the form $\mathsf{Value}(v_i,b, [0,\now],1)$ can appear in the intersection of the repairs in $\reps(\mathsf{SE}(\dataset_\varphi, \eventspec), \Sigma)$, and thus cannot appear in the cautious timeline. 
This means that the intersection of the repairs in $\reps(\mathsf{SE}(\dataset_\varphi, \eventspec), \Sigma)$ 
is contained in $\{ \mathsf{Q}([0,\now], 1) \}$. Further note that no meta-event rule is applicable in 
the absence of $\mathsf{Value}$ facts, i.e.\ $\mathsf{ME}(\dataset_\varphi, {\mathsf{Q}([0,\now], 1)}, \Sigma)) = \emptyset$. Given our assumption that 
$\{ \mathsf{Q}([0,\now], 1) \}$ is not the cautious timeline, it follows that the intersection of repairs in
$\reps(\mathsf{SE}(\dataset_\varphi, \eventspec), \Sigma)$ yields the empty set. However, this means 
that there must exist a repair $\mathcal{R} \in \reps(\mathsf{SE}(\dataset_\varphi, \eventspec), \Sigma)$
such that $\mathsf{Q}([0,\now], 1) \not \in \mathcal{R}$. Due to the maximality of repairs and the absence 
of negated atoms (cf.\ monotonicity property discussed in proof of Theorem \ref{thm:ptime-noneg}),
this means that adding $\mathsf{Q}([0,\now], 1)$ would result in a constraint violation. 
However, this can only happen if $\mathcal{R}$ contains $\mathsf{Sat}(c_m,[0,\now])$, 
which due to the definition of the rules in $\metarules$ implies that all of the clauses in $\varphi$ are satisfied. 
It follows that $\varphi$ is satisfiable, with a satisfying valuation defined by the $\mathsf{Value}$ facts 
retained in $\mathcal{R}$. \smallskip\\
\noindent$(\Leftarrow)$. Suppose that $\varphi$ is satisfiable, with satisfying valuation $\mu$. 
Define $\mathcal{R} \subseteq \mathsf{SE}(\dataset_\varphi, \eventspec)$ as follows:
\begin{align*}
\mathcal{R}= & \{\mathsf{Value}(v_i,\mu(v_i), [0,\now],1) \mid 1 \leq i \leq k \}
\end{align*}
It is easily verified that $\mathcal{R}$ is $\eventspec$-consistent, since there is a single $\mathsf{Value}$ fact
per $v_i$ and the fact $\mathsf{Q}([0,\now],1)$ is absent. It is also maximally consistent. Indeed: 
\begin{itemize}
\item No further $\mathsf{Value}$ fact from $\mathsf{SE}(\dataset_\varphi, \eventspec)$ can be added, else the first constraint is violated
\item Due to the fact that the $\mathsf{Value}$ facts in $\mathcal{R}$ assign truth values to variables according to the satisfying valuation $\mu$, the fact $\mathsf{Sat}(c_m,[0,\now])$ can be derived using the meta-rules, and so adding $\mathsf{Q}([0,\now],1)$ would lead to a violation of the second constraint. 
\end{itemize}
It follows that $\mathsf{Q}([0,\now],1)$ does not belong to the intersection of the repairs in $\reps(\mathsf{SE}(\dataset_\varphi, \eventspec), \Sigma)$, hence it is not possible for $\{ \mathsf{Q}([0,\now], 1) \}$ to be the cautious timeline.
\end{proof}

We observe that the preceding reduction crucially relies upon using
recursion in the meta-rules. It would therefore be relevant to consider 
TESs with only non-recursive meta-event rules (together with additional 
restrictions, like no negated event atoms) to identify fragments for 
which the cautious timeline can be tractably computed.

\medskip

In preparation for Theorem \ref{thm:onlytemp}, we prove the following lemma,
which clarifies the possible relationships between intervals
associated with the same $R(\vect{d})$, in the case of restricted specifications.
In particular, it implies that if 
a lower confidence interval non-trivially overlaps with a higher confidence interval,
it must fully contain it.

\begin{lemma}\label{nest-lemma}
Consider a TES $\eventspec$ such that $\domconstrain= \emptyset$ 
and termination rules all have confidence $1$, and 
suppose that $\simplerules, \dataset \models_\ell R(\vect{d} , [t_1,t_2])$
and $\simplerules, \dataset \models_{\ell'} R(\vect{d} , [t_1',t_2'])$. Then:
\begin{enumerate}
\item If $\ell = \ell'$ and $ [t_1',t_2'] \neq  [t_1,t_2]$, then $t_1 \leq t_2 \leq t_1' \leq t_2'$ or 
$t_1' \leq t_2' \leq t_1\leq t_2$ 
\item If $\ell' > \ell$, then $ [t_1',t_2'] \neq  [t_1,t_2]$ and we cannot have $t_1 <t_1'<  t_2 \leq t_2'$ nor 
$t_1' \leq t_1 < t_2' < t_2$
\end{enumerate}
\end{lemma}
\begin{proof}
We give the proof for non-persistent events, the argument for persistent events is similar but simpler.

To show point 1, suppose we have $\simplerules, \dataset \models_\ell R(\vect{d}, [t_1,t_2])$
and $\simplerules, \dataset \models_{\ell} R(\vect{d} , [t_1',t_2'])$ with $ [t_1',t_2'] \neq  [t_1,t_2]$. 
We may suppose w.l.o.g.\ that $t_1 \leq t_1'$. We thus aim to show that $t_1 \leq t_2 \leq t_1' \leq t_2'$. 
Assume for a contradiction that $t_1 \leq t_1' < t_2$. 
If $t_1=t_1'$, then $t_2 \neq t_2'$ (as we know $ [t_1',t_2'] \neq  [t_1,t_2]$), 
which implies that one of the intervals could have been further extended, 
violating one of the conditions of Definition \ref{nspred-def} (Items~4–5). 
Thus we have $t_1 < t_1' < t_2$. However, this also yields a contradiction, since 
whichever timepoints in $T_\exists^\ell(R(\vect{d}))$ were used to validate 
item 2 of Definition \ref{nspred-def} to witness that $\simplerules, \dataset \models_\ell R(\vect{d}, [t_1,t_2])$
could also be used to show that $[t_1',t_2']$ does \emph{not} 
verify item 4, contradicting our assumption that $\simplerules, \dataset \models_{\ell} R(\vect{d} , [t_1',t_2'])$. 
Thus, it must be the case that $t_1 \leq t_2 \leq t_1' \leq t_2'$. 

To show point 2, suppose we have $\simplerules, \dataset \models_\ell R(\vect{d} , [t_1,t_2])$
and $\simplerules, \dataset \models_{\ell'} R(\vect{d} , [t_1',t_2'])$ with $\ell' > \ell$. 
By item 7 of Definition \ref{nspred-def}, we directly get $[t_1',t_2'] \neq  [t_1,t_2]$. 
Suppose for a contradiction that $t_1 <t_1'<  t_2 \leq t_2'$. 
Then the timepoints in $T_\exists^\ell(R(\vect{d}))$ which were used to validate 
item 2 of Definition \ref{nspred-def} for $ [t_1,t_2]$ are also present in $T_\exists^{\ell'}(R(\vect{d}))$ (since $\ell'>\ell$)
and so can be used to show that $[t_1',t_2']$ does not verify item 4 (as an earlier start is possible). 
However, we must also argue that there is no $t_e \in T_\stop^{\ell'}(R(\vect{d}))$
that could block such an extension. It is here that we must use the fact that 
termination rules all have confidence 1, which means in particular that 
$T_\stop^\ell(R(\vect{d})) = T_\stop^{\ell'}(R(\vect{d}))$. Since there was no 
blocking termination timepoint at level $\ell$, there cannot be any such timepoint w.r.t.\ $\ell'$.
Now suppose for a contradiction that $t_1' \leq t_1 < t_2' < t_2$. 
Then we can use the timepoints in $T_\exists^\ell(R(\vect{d}))$ that permit the interval to continue $t_2$
to show that we could have chosen a later end for $[t_1',t_2']$. Here again we use the
assumption that termination rules all have confidence 1 to infer that there is no termination timepoint 
that can block this extension. 
We thus obtain the desired contradiction. 
\end{proof}

\noindent \textbf{Theorem \ref{thm:onlytemp}}. 
When $\domconstrain= \emptyset$ and termination rules all have confidence $1$, 
there is a unique preferred repair, 
and both the preferred timeline and cautious timeline can be computed in \textsc{PTime} in data complexity. 
\begin{proof}
Let $\eventspec = (\simplerules, \metarules, \tconstrain, \domconstrain)$ be a TES satisfying the conditions of the statement, 
i.e.\ $\domconstrain= \emptyset$ and all termination rules in $\simplerules$ have the same confidence level of~1. 

For the cautious timeline, we simply note that a fact $R(\vect{d} , [t_1,t_2], \ell) \in \simplerules(\dataset)$ will belong to the intersection of all repairs 
just in the case that there does not exist another fact $R(\vect{d} , [t_1',t_2'], \ell') \in \simplerules(\dataset)$ such that 
$R(\vect{d} , [t_1,t_2])$ and $R(\vect{d} , [t_1',t_2'])$ together violate one of the temporal constraints. We can thus iterate over all (polynomially many) 
pairs of facts and remove those that participate in at least one constraint violation. The remaining facts give us the intersection of repairs, from which we can construct the cautious timeline, by applying the meta-event rules.

Next, we aim to show that there is a unique preferred repair. To this end, let $\mathcal{S}= \simplerules(\dataset)$, and consider the repair $\mathcal{R}^*$ constructed greedily as follows (and which is formalized in Algorithm \ref{alg:theorem3_preferred_timeline}): 
\begin{itemize}
\item Initialize $\mathcal{R}^*$ with all facts in $\mathcal{S}_m$, where $m$ is the minimum level appearing in $\mathcal{S}$
\item For each $\ell$ from $m+1$ to $n$ (with $n$ the maximum level in $\mathcal{S}$): 
add $R(\vect{d} , [t_1,t_2], \ell) \in \mathcal{S}_\ell$ to $\mathcal{R}^*$ if there is no $R(\vect{d} , [t_1',t_2'], \ell') \in \mathcal{R}^*$ such that $R(\vect{d} , [t_1',t_2'])$ and $R(\vect{d} , [t_1,t_2])$ 
violate a temporal constraint in $\tconstrain$. 
\end{itemize}
Note that $\mathcal{S}_m$ is $\eventspec$-consistent, since by Lemma \ref{nest-lemma}, there cannot exist annotated event facts 
$R(\vect{d} , [t_1,t_2], m)$ and $R(\vect{d} , [t_1',t_2'], m)$ whose intervals non-trivially overlap. Moreover, the same will hold within each single confidence level. Moreover, due to the way we define $\mathcal{R}^*$, we will never add facts which are in conflict with a fact already selected. 
Thus, $\mathcal{R}^*$ is $\eventspec$-consistent. It is also maximal, as every fact that is excluded would introduce a constraint violation. Furthermore, due to the level-by-level construction (starting from the best confidence level), $\mathcal{R}^*$ is a preferred repair. Finally, we note that due to the lack of constraint violations within a given level (due to Lemma \ref{nest-lemma}), there is never a decision as to which facts from a given level can be added. Thus, 
$\mathcal{R}^*$ is the only preferred repair, and it is clear from its definition that it can be computed in polynomial time (w.r.t.\ data complexity). \end{proof}

\section{HEVA System}\label{ap:casper_system}

The diagram (Figure~\ref{casper_architecture}) provides an overview of the \casper architecture and its main components. The top section shows the three main inputs to \casper. In the yellow box, users define temporal event rules and atemporal facts. Observation facts, shown in the purple box, are automatically generated via an external Python script using a mappings file that links observation predicates to fields in a relational database.
The central blue box represents the \casper system itself, composed of several modules: the non-persistent simple event module, persistent simple event module, temporal predicate module, and auxiliary module. An additional temporal repair module can be activated if the repair option is enabled.
Finally, the bottom red box shows the different types of output produced by \casper, corresponding to various timelines (see Section~\ref{semantics}, Definition 9): the naïve timeline, consistent timelines, preferred timeline, and the cautious timeline.

We now provide a more detailed description of the non-persistent simple event, persistent simple event, temporal predicate and temporal repair modules.
\begin{figure}[ht]
    \centering
    \includegraphics[width=0.46\textwidth]{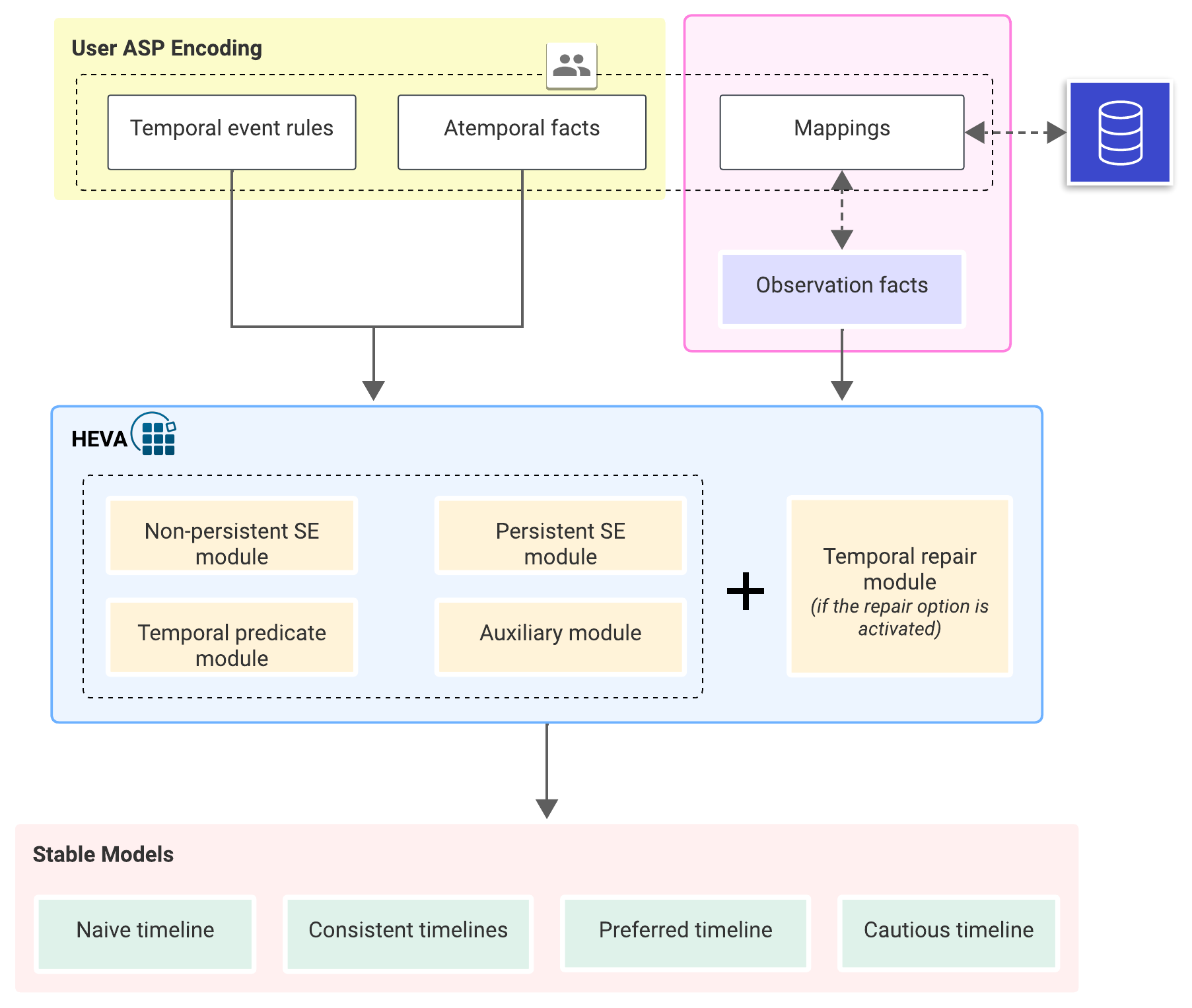}
    \caption{
    Overview of the \casper system architecture.
    }
    \label{casper_architecture}
\end{figure}

\subsection{Non-persistent Simple Event Module}

This module computes inferred \emph{non-persistent simple event facts}
from the predicates \texttt{exists} and \texttt{terminates},
following the formal construction of Definition~\ref{nspred-def}.
The computation proceeds by iterative interval construction 
with level-aware expansion and pruning.

\paragraph{Initialization.}
For a non-persistent event predicate $R \in \nspreds$ and arguments $\vect{d}$,
we collect all timepoints
$T^\ell_{\exists}(R(d))$ at which an existence condition
holds with confidence level $\ell$.
Each such timepoint $t$ initializes a singleton interval $[t,t]$,
encoded by the predicate \texttt{pre\_candidate} (see line~2 of Listing~\ref{lst:casper_code}). If multiple existence facts occur at the same timepoint,
the minimal (best) confidence level is selected.

\paragraph{Level-wise expansion.}
Interval expansion is performed hierarchically by confidence level,
starting from the highest-priority (lowest numerical) level.

Let $[t_1,t_2]$ be a current candidate interval at level $\ell$,
and let $w$ be the temporal window given by \texttt{pt\_window}
(corresponding to $\windowpred(R(\vect{d}),w)$).

\begin{enumerate}
    \item \textbf{Forward expansion within the window.}
    The interval is extended to $[t_1,t_2']$ if there exists
    $t_2' \in T^\ell_{\exists}(R(d))$ such that
    \[
        t_2 < t_2' \leq t_2 + w
    \]
    and there is no termination timepoint
    $t \in T^\ell_{\times}(R(d))$ with
    $t_2 \leq t \leq t_2'$.
    The right boundary is chosen as the maximal such $t_2'$.

    \item \textbf{Termination-aware expansion.}
    If a termination timepoint $t_b \in T^\ell_{\times}(R(d))$
    satisfies
    \[
        t_2 < t_b \leq t_2 + w,
    \]
    the interval is closed at $t_b$.
    The earliest compatible termination is selected.

    \item \textbf{Lower-priority levels.}
    After processing the highest-priority level,
    the procedure continues with lower-confidence levels.
    At level $\ell$, only existence and termination facts
    with confidence $\leq \ell$ are considered.
    This ensures that stronger evidence dominates weaker evidence,
    while still allowing consistent lower-confidence observations
    to extend intervals.
\end{enumerate}

\paragraph{Pruning non-maximal intervals.}
During expansion, multiple overlapping or nested candidates may arise.
The predicates \texttt{pre\_covered} and \texttt{covered}
eliminate non-maximal intervals.

\paragraph{Final event construction.}
The final \texttt{event} facts are obtained by:
(i) selecting intervals for which no strictly better
confidence level exists, and
(ii) discarding intervals still covered by longer candidates.

Overall, this module implements the formal construction
of non-persistent simple events by initializing singleton
intervals from existence timepoints, expanding them within
a bounded temporal window, respecting termination conditions,
propagating confidence levels hierarchically, and retaining
only maximal consistent intervals.

\begin{figure*}[h!]
\begin{lstlisting}[style=aspstyle, caption={Some lines of source code from \casper}, label={lst:casper_code}]
% pre-event candidate initialization
pre_candidate(N,P,E,(Tx,Tx),L) :- exists(N,P,E,Tx,L).

% initialization of persistent simple event
l_candidate(N, P, E, (Tx, Tx), L) :- exists_pers(N, P, E, Tx, L).

% repair mechanism
{ rep_event(ID,N,P,E,(Start,End),L) } :- event(ID,N,P,E,(Start,End),L).

:- rep_event(_,N,P,E,(Start1,End1),_), rep_event(_,N,P,E,(Start2,End2),_),
    contains((Start1,End1),(Start2,End2)).

inconsistent_if_added(N,P,E,(Start2,End2)) :- rep_event(_,N,P,E,(Start1,End1),_),
    contains((Start1,End1),(Start2,End2)).

:- not rep_event(_,N,P,E,(Start,End),_), event(_,N,P,E,(Start,End),_),
    not inconsistent_if_added(N,P,E,(Start,End)).

% preferred repair
temporal_conflict_with_lower_level(N, P, E, (Start2, End2), L2) :-
    event(_, N, P, E, (Start1, End1), L1), event(_, N, P, E, (Start2, End2), L2),
    L1 < L2, Start1 < Start2, Start2 < End1.
...

rep_event(ID, N, P, E, (Start, End), L) :- event(ID, N, P, E, (Start, End), L),
    not temporal_conflict_with_lower_level(N, P, E, (Start, End), L).
\end{lstlisting}
\noindent
Listing~\ref{lst:casper_code} shows a simplified fragment of the ASP code implementing the repair mechanism in \casper. The first choice rule allows the solver to freely select which inferred events (\texttt{event}) are retained in the repaired set (\texttt{rep\_event}).  
The following integrity constraints ensure temporal consistency by forbidding answer sets where one event interval \emph{contains} another interval of the same type, thereby enforcing non-overlapping intervals (cf.\ Section~\ref{specifications}).  
The predicate \texttt{inconsistent\_if\_added} marks events that would violate this property if included.  
The last part of the listing implements the computation of the \emph{preferred repair}. The predicate \texttt{temporal\_conflict\_with\_lower\_level} detects events whose intervals overlap with an event of the same type inferred at a strictly lower confidence level. Such events are discarded, ensuring that higher-confidence events are always preferred over conflicting lower-confidence ones. As a result, \texttt{rep\_event} retains exactly those events that do not introduce temporal conflicts with events of lower confidence, yielding a unique preferred repair consistent with the lexicographic preference over confidence levels defined in the framework.
\end{figure*}

\subsection{Persistent Simple Event Module}

This module computes inferred \emph{persistent simple event facts},
corresponding to events that, once initiated, persist until an
explicit termination condition occurs.
It relies on the predi\-cates \texttt{exists\_pers} and
\texttt{terminates}, and constructs maximal intervals
consistent with confidence levels.

\paragraph{Initialization.}
For a persistent event predicate $R \in \pspreds$ and arguments $\vect{d}$,
each timepoint
$t_1 \in T^\ell_{\exists}(R(\vect{d}))$
initializes a singleton candidate interval $[t_1,t_1]$,
encoded by the predicate \texttt{l\_candidate} (see line~5 of Listing~\ref{lst:casper_code}).
If multiple existence facts occur at the same timepoint,
the minimal (best) confidence level is selected.

\paragraph{Filtering by higher-confidence intervals.}
Before constructing an interval from $t_1$,
the module checks whether $t_1$ is already covered
by an interval inferred at a strictly better
confidence level.
Formally, if there exists a candidate interval
$[t'_1,t'_2]$ with confidence $\ell' < \ell$
such that $t'_1 \leq t_1 \leq t'_2$,
then $t_1$ is discarded.
This is implemented by the auxiliary predicate
\texttt{pe\_covered\_by\_higher}.

\paragraph{Determining the effective start.}
If $t_1$ is not covered, the module determines whether
a termination timepoint
$t_b \in T^\ell_{\times}(R(\vect{d}))$
occurs strictly before $t_1$.
When such a $t_b$ exists, only the most recent one
is considered.
If another existence timepoint lies between $t_b$
and $t_1$, the candidate is discarded in order to avoid
generating non-maximal intervals.
Otherwise, the true start of the interval is set to the
earliest existence timepoint following $t_b$.

If no prior termination exists, the module verifies that
no earlier uncovered existence timepoint precedes $t_1$.
This ensures that only minimal valid starting points
are retained.

\paragraph{Determining the end.}
For each valid start timepoint $t_1$,
the earliest termination timepoint $t_2 > t_1$
with confidence level $\leq \ell$
is selected.
If such a termination exists, the interval
$[t_1,t_2]$ is constructed.

If no termination timepoint follows $t_1$,
the event is considered \emph{ongoing}.
In this case, the interval is assigned an open end
(e.g., $[t_1,\ast]$ or $[t_1,t_c{+}1]$,
where $t_c$ denotes the current time).
This allows the system to distinguish explicitly
between ongoing events and those bounded by
an observed termination.

\paragraph{Maximality and final construction.}
The predicate \texttt{pe\_candidate}
collects all intervals satisfying the above conditions.
As persistent events do not rely on temporal windows,
interval construction is linear with respect to
termination ordering.
The final \texttt{event} facts are directly generated
from \texttt{pe\_candidate}, yielding maximal,
confidence-aware persistent intervals.

Overall, the persistent simple event module implements
the formal semantics of persistent events by:
(i) initializing intervals from existence timepoints,
(ii) preventing overlap with higher-confidence intervals,
(iii) identifying the most recent blocking termination,
(iv) selecting the earliest compatible future termination,
and (v) producing maximal intervals that remain valid
until explicit termination.

\subsection{Temporal Predicate and Auxiliary Modules}  
This module provides the temporal reasoning primitives used throughout \casper.  
It defines the standard relations of Allen’s interval algebra (e.g.\ \texttt{before}, \texttt{meets}, \texttt{overlaps}, \texttt{contains}) for interval-interval reasoning, as well as Vilain’s point–interval algebra for point-based reasoning.  
It also introduces interval manipulation predicates such as \texttt{intersection\_of} (to compute the intersection of intervals), \texttt{union\_of} (for interval union), and helper predicates that simplify rule writing like \texttt{start} and \texttt{end} to extract the earliest and latest timepoints of a given event.  
These predicates are available to users for defining meta-event rules, while additional internal predicates are provided by the \textbf{auxiliary module} to simplify reasoning and optimize performance.  
For example, some auxiliary rules act as control switches that, when triggered, selectively prevent the grounding of certain rule patterns, thereby reducing the size of the grounded program and improving solving efficiency.

\subsection{Temporal Repair Module} This module computes the set of repairs—or, when applicable, the unique preferred repair—of the inferred simple events.  
By default, \casper\ performs no repair; this module is activated only when explicitly requested at runtime.  
It implements the notions of consistency and repair introduced in Section~\ref{semantics} (Definitions~\ref{tes-cons}, \ref{simplerep}, and~\ref{prefrepdef}) using the fixed set of temporal constraints $\tconstrain$ from Section~\ref{specifications}. The ASP encoding relies on choice rules and leverages the temporal relations defined in the temporal predicate—such as \texttt{overlaps} or \texttt{contains}—to specify constraints more precisely and to check for their violations. 

The computation of preferred repairs follows a level-wise
selection strategy: events are considered in increasing order
of confidence level, and at each level, events that do not
introduce temporal conflicts with already selected events
are retained.
For the restricted fragment of TESs supported by \casper\
(i.e., no domain constraints and all termination rules with
confidence~1), this procedure coincides with the formal
definition of preferred repairs and guarantees both:
\begin{itemize}
    \item the existence of a unique preferred repair, and
    \item polynomial-time computability,
\end{itemize}
as established in Theorem~\ref{thm:onlytemp}.

Listing~\ref{lst:casper_code} displays selected (simplified) lines from the ASP encoding to compute (preferrred) repairs.

\section{Lung Cancer Use Case}\label{lung_case}

\subsection{Medical Background}
According to the natural history of cancer, the disease typically begins with localized cell proliferation in a specific organ tissue, called the \textit{primary tumor}. Malignant cells can then spread to distant sites, forming \textit{secondary tumors}, which retain the histological characteristics of the primary tumor. Clinically, cancer progresses through a series of episodes, interspersed with periods of \textit{remission}, defined as either a partial reduction or complete disappearance of symptoms. After remission, the same cancer may recur at the original site or elsewhere; such a recurrence is called a \textit{relapse}. Cancer treatment is multimodal, encompassing surgery, radiotherapy, immunotherapy, chemotherapy and/or targeted therapy. These modalities are often combined and personalized according to tumor type and patient characteristics. For example, \textit{targeted therapies with tyrosine kinase inhibitors (TKIs)} are used in certain subtypes of lung cancer harboring specific mutations, such as \textit{EGFR or ALK mutations}. Each cancer episode generally corresponds to one or more \textit{lines of treatment}, a line designating a set of therapies applied sequentially or in combination. A new line of treatment is introduced when the previous one proves ineffective.

\subsection{Lung Cancer Events}\label{lg_appendix}

\begin{table}[H]
\centering
\caption{Summary of inferred events in the lung cancer use case: Simple Event (SE) - Meta-Event (ME).}\label{lung_event}
        \begin{tabular}{ll}
          \textbf{Event} & \textbf{Type} \\
          \midrule
          Primary lung cancer episode & Non Persistent SE \\
          Secondary cancer episode & Persistent SE \\
          Presence of EGFR/ALK mutation & Persistent SE \\
          TKI targeted therapy & Persistent SE \\
          Lung cancer disease & ME \\
          \bottomrule
        \end{tabular}
\end{table}

Following the summary of the inferred event types in Table~\ref{lung_event}, we now provide excerpts of the concrete encoding used in the lung cancer use case. The next listings illustrate how the different categories of events—simple events and meta-events—are instantiated within \casper using ASP rules and domain knowledge.

Listing~\ref{lg:se} presents the temporal rules defining the simple events (both non-persistent and persistent), including their existence and termination conditions, as well as the temporal window used for interval construction. These rules directly implement the formal notions introduced in Section~\ref{specifications} for specifying simple events from timestamped observations.

Listing~\ref{lg:me} illustrates the definition of meta-events, showing how higher-level clinical constructs are derived from previously inferred simple events through interval relations and confidence propagation. This exemplifies the stratified meta-event mechanism described in the formal framework.

Finally, Listing~\ref{ls:atemp_lg} provides selected atemporal facts encoding domain knowledge (e.g., ADICAP codes, ICD-10 categories, and TKI drugs) used by the rules. These facts serve as the semantic backbone linking structured EHR observations to clinically meaningful event definitions. The corresponding ASP encodings are provided below:

\clearpage

\begin{figure*}[t]
\begin{lstlisting}[style=aspstyle, caption={Temporal event rules for selected simple events in the lung cancer use case.}, label={lg:se}]
% primary lung cancer episode (non-persistent; window below)
exists(primary_lungC_episode,P,c34,T,1) :- obs(has_lungc_adicap_diag,P,E,T).
exists(primary_lungC_episode,P,c34,T,2) :- obs(has_icd10_diag,P,E,T), E != c349, E != c348,
    primary_lungC(E).
exists(primary_lungC_episode,P,c34,T,3) :- obs(has_icd10_diag,P,c348,T).
exists(primary_lungC_episode,P,c34,T,3) :- obs(has_icd10_diag,P,c349,T).

% secondary cancer episode (persistent)
exists_pers(secondary_cancer_episode,P,E,T,2) :- 
    obs(has_icd10_diag,P,E,T), secondary_cancer(E).

% TKI therapy (persistent; admin more reliable than prescription)
exists_pers(tki_target_therapy,P,D,T,1) :- obs(has_adm,P,D,T), tki(D).
terminates(tki_target_therapy,P,D,T,1) :- obs(has_adm, P, D1, T), obs(has_adm, P, D, T1),
    tki(D1), tki(D), D != D1, T1 < T.

exists_pers(tki_target_therapy,P,D,T,2) :- obs(has_presc,P,D,T), tki(D).
terminates(tki_target_therapy,P,D,T,2) :- obs(has_presc, P, D1, T), obs(has_presc, P, D, T1),
    tki(D1), tki(D), D != D1, T1 < T.

% EGFR mutation (persistent)
exists_pers(egfr_mutation,P,egfr,T,1) :- obs(has_egfr_mut,P,E,T).
terminates(egfr_mutation,P,egfr,T,1) :- obs(has_no_egfr_mut,P,E,T).

% ALK mutation (persistent)
exists_pers(alk_mutation,P,alk,T,1) :- obs(has_alk_mut,P,E,T).
terminates(alk_mutation,P,alk,T,1) :- obs(has_no_alk_mut,P,E,T).

% non-persistent window (time in seconds since Unix epoch)
pt_window(primary_lungC_episode,c34,48988800). % 18 months
\end{lstlisting}
\noindent
Listing~\ref{lg:se} defines the existential and termination rules used by \casper\ to infer key simple events in the lung cancer use case. Here, \texttt{P} denotes the patient, \texttt{E} an observed clinical \emph{entity} (e.g. diagnosis code or mutation marker), and \texttt{D} a drug. The constant \texttt{c34} represents the ICD-10 category “malignant neoplasm of bronchus and lung”, serving as the general entity for a \emph{primary lung cancer episode}. Time is represented in seconds since the Unix epoch. The first four rules define the non-persistent \texttt{primary\_lungC\_episode}: an event of confidence level~1 is inferred from ADICAP-coded diagnoses (\texttt{has\_lungc\_adicap\_diag}), while lower-confidence alternatives rely on ICD-10 codes—specific lung codes excluding \texttt{C34.8} and \texttt{C34.9} (\(\ell=2\)), then generic unspecified codes (\(\ell=3\)). The secondary cancer episode is modeled as a persistent event inferred when an ICD-10 diagnosis corresponds to a metastasis (\texttt{secondary\_cancer(E)}). TKI targeted therapies are also persistent: administration records (\(\ell=1\)) and prescription records (\(\ell=2\)) indicate therapy existence, while observing another TKI drug (\texttt{D1 != D}) signals termination, marking a treatment switch. EGFR and ALK mutations are persistent events detected from genetic test results (\texttt{has\_egfr\_mut} / \texttt{has\_alk\_mut}) and terminated when a corresponding negative observation appears (\texttt{has\_no\_...}). Finally, the \texttt{pt\_window} sets the expansion window for the non-persistent primary episode (\(48{,}988{,}800\) seconds \(\approx\) 18 months). \bigskip
\end{figure*}

\begin{figure*}[t]
\begin{lstlisting}[style=aspstyle, caption={Example of meta-event rules.}, label={lg:me}]
% lung cancer disease
m_event(lung_cancer_disease,P,(T1,T2),L) :-
    start(primary_lungC_episode,P,T1,L);
    not event(_,secondary_cancer_episode,P,_,(Tx,_),_) : Tx < T1, valid_time(Tx);
    persist_end(P,T2).

m_event(lung_cancer_disease,P,(T1,T2),L) :-
    start(secondary_cancer_episode,P,T1,L);
    not event(_,primary_lungC_episode,P,_,(Tx,_),_) : Tx <= T1, valid_time(Tx);
    persist_end(P,T2).

% targeted mutation = overlap of disease and TKI therapy (confidence via #min)
m_event(targeted_mutation,P,(T1,T2),L) :-
    event(_,lung_cancer_disease,P,(T1x,T2x),_),
    event(_,tki_target_therapy,P,_,(T1y,T2y),_),
    intersection_of((T1x,T2x),(T1y,T2y),(T1,T2)),
    L = #min{Lx : event(_,lung_cancer_disease,P,(T1x,T2x),Lx);
             Ly : event(_,tki_target_therapy,P,_,(T1y,T2y),Ly)}.
\end{lstlisting}
\noindent
Listing~\ref{lg:me} specifies meta-events. Underscores \texttt{\_} mark \emph{irrelevant arguments} that are intentionally ignored (e.g., internal IDs or entities not needed by the rule). The first rule builds \texttt{lung\_cancer\_disease} starting from the earliest \texttt{primary\_lungC\_episode} (a simple event): \texttt{start(primary\_lungC\_episode,P,T1,L)} retrieves the earliest start; the conditional literal syntax \texttt{not event(\_,secondary\_cancer\_episode,P,\_,(Tx,\_),\_) : Tx < T1, valid\_time(Tx)} means “for every time \texttt{Tx} earlier than \texttt{T1} that is a valid time, there must be \emph{no} secondary cancer event”, i.e. the rule requires the absence of any prior secondary episode. Here, the part after the colon \texttt{: Tx < T1, valid\_time(Tx)} is a \emph{condition} that ranges the variable \texttt{Tx}. The helper predicate \texttt{persist\_end} is used to indicate that the end time of an event is ongoing. The second rule symmetrically allows \texttt{lung\_cancer\_disease} to start from a \texttt{secondary\_cancer\_episode} provided that no \texttt{primary\_lungC\_episode} occurred strictly earlier (\texttt{Tx <= T1}). The last rule defines \texttt{targeted\_mutation} as the \emph{overlap} between \texttt{lung\_cancer\_disease} and \texttt{tki\_target\_therapy}: the rule retrieves the time intervals of both the lung cancer disease and the TKI therapy for the same patient, then computes their temporal overlap to define the period during which both are active. The resulting meta-event, inherits its confidence level as the minimum of the two contributing events’ confidences. It is important to note that the last rule is provided as an example illustrating the use of \texttt{intersect\_of} to compute the intersection of two intervals. This meta-event is not part of the target events relevant to our use case.

\vspace{1em}

\begin{lstlisting}[style=aspstyle, caption={Atemporal facts encoding domain knowledge.}, label={ls:atemp_lg}]
% ADICAP histopathology codes indicating primary lung cancer
lung_cancer_adicap(rbe7a0).
lung_cancer_adicap(rba7v4).
lung_cancer_adicap(rba7a0).
...

% TKI drugs
tki(ceritinib).
tki(osimertinib).
tki(crizotinib).
...

% Secondary cancer (ICD-10) examples
secondary_cancer(c711).
secondary_cancer(c712).
secondary_cancer(c716).
...
\end{lstlisting}
\noindent
\textit{Knowledge sources.}
Atemporal facts in Listing~\ref{ls:atemp_lg}, encoding domain knowledge relevant to this use case draw from various medical thesauri such as ADICAP, ICD-10, and ATC.
\end{figure*}

\FloatBarrier

\subsection{Execution Time of Three Other CASPER Timelines}

In the main paper, we only provided the execution times for the mode that produces consistent timelines. Here we  provide the tables of execution times when \casper\ is run in other three modes (corresponding to naïve, cautious, and preferred timelines). As shown in these three tables, \ncasper's execution time remains similar across the different timelines.
    
    \begin{table}[H]
    \setlength{\tabcolsep}{1.2mm}
    \centering
        \caption{
       Statistics on execution times for computing all stable models associated with naïve timeline in the lung cancer use case (322 patients).}
    \label{execution_stats_naive_lung_cancer}
        \begin{tabular}{lcccccc}
        & Min. & Q1 & Q2 & Mean & Q3 & Max.\\ \midrule
 \textbf{Grnd.\ rules} & 538 & 673.2 & 793.5 & 928.5 & 1,013 & 3,858\\
\textbf{Time (s)} & 0.08 & 0.09 & 0.09 & 0.14 & 0.10 & 5.38 \\ \bottomrule
        \end{tabular}
    \end{table}
    
    \begin{table}[H]
    \setlength{\tabcolsep}{1.2mm}
    \centering
        \caption{
       Statistics on execution times for computing all stable models associated with preferred timeline in the lung cancer use case (322 patients).}
    \label{execution_stats_pref_lung_cancer}
        \begin{tabular}{lcccccc}
        & Min. & Q1 & Q2 & Mean & Q3 & Max.\\ \midrule
\textbf{Grnd.\ rules} & 534 & 679.2 & 800.5 & 935.5 & 1,018.8 & 3,844\\
\textbf{Time (s)} & 0.14 & 0.15 & 0.15 & 0.20 & 0.16 & 4.69 \\ \bottomrule
        \end{tabular}
    \end{table}
 
    \begin{table}[H]
    \setlength{\tabcolsep}{1.2mm}
    \centering
        \caption{
       Statistics on execution times for computing all stable models associated with cautious timeline in the lung cancer use case (322 patients).}
    \label{execution_stats_cautious_lung_cancer}
        \begin{tabular}{lcccccc}
        & Min. & Q1 & Q2 & Mean & Q3 & Max.\\ \midrule
\textbf{Grnd.\ rules} & 538 & 708 & 841.5 & 983.9 & 1,066.5 & 3,918\\
\textbf{Time (s)} & 0.15 & 0.22 & 0.24 & 0.30 & 0.31 & 5.04 \\ \bottomrule
        \end{tabular}
    \end{table}

\subsection{Annotation Agreement Evaluation}\label{evaluation}

\subsubsection{Inter-annotator agreement.}
To assess the consistency of clinical event annotations across annotator pairs, we computed a weighted inter-annotator agreement based on the number of shared elements identified for each annotated event.
The annotated events included primary lung cancer episodes, secondary cancer episodes, EGFR/ALK mutation status, TKI-targeted therapies, and the presence of lung cancer disease. Agreement was assessed using a point-based scheme that considers the individual components of each event. A maximum score was defined for each event type, based on its annotation granularity:

\begin{itemize}
    \item for primary lung cancer episode, secondary cancer episode, TKI targeted therapy: up to 4 points were assigned:
    \begin{itemize}
        \item 1 point for the \textbf{presence} of the event,
        \item 1 point for the \textbf{specific subtype of event} (e.g., type of therapy),
        \item 1 point for the \textbf{start date},
        \item 1 point for the \textbf{end date}.
    \end{itemize}
    \item for \textbf{other events} (e.g. \textit{EGFR/ALK mutation}, \textit{lung cancer disease}), up to 3 points were assigned:
    \begin{itemize}
        \item 1 point for the \textbf{presence} of the event,
        \item 1 point for the \textbf{start date},
        \item 1 point for the \textbf{end date}.
    \end{itemize}
\end{itemize}

Note that the “Subtype” component was not applicable for EGFR/ALK mutation and lung cancer disease events. In the case of EGFR and ALK mutations, the mutation name itself (EGFR or ALK) inherently defines the subtype, making a separate annotation unnecessary. For lung cancer disease, histological subtypes (e.g., adenocarcinoma or squamous cell carcinoma) were not annotated, as they were not considered relevant for this particular use case. If multiple occurrences of the same event were detected and one or more of these occurrences did not appear in an annotator’s records, such cases were counted as disagreements, with an agreement score of 0 assigned.

Let \( A_{e,i} \) denote the number of agreement points awarded for an annotated instance \(i\) of event type \(e\), 
and \( C_{e,i} \) the corresponding maximum possible score.
The agreement ratio for each annotation is then computed as follows:
\[
\text{Agreement Ratio}_{e,i} = \frac{A_{e,i}}{C_{e,i}}.
\]

For each event type \(e\), the mean agreement ratio is first computed over its \(n_e\) annotated instances (Table~\ref{evaluation_agreement}):

\[
\text{Mean Agreement per Event Type}_e = 
\frac{1}{n_e} \sum_{i=1}^{n_e} \frac{A_{e,i}}{C_{e,i}}
\]

To avoid bias toward event types with more annotated instances,  
the overall inter-annotator agreement is defined as the \emph{unweighted mean} of these per-event-type means,  
giving equal importance to each event type:

\[
\text{Mean Inter-Annotator Agreement} = 
\frac{1}{k} \sum_{e=1}^{k} 
\left(
\frac{1}{n_e} \sum_{i=1}^{n_e} \frac{A_{e,i}}{C_{e,i}}
\right)
\]
where \(k\) is the number of distinct event types and \(n_e\) the number of annotated instances for each type.  
Using this formula, the mean inter-annotator agreement is \textbf{60.56\%}.

We adopt this unweighted formulation because our goal is to evaluate \casper\!\!\!’s ability to infer diverse types of clinical events, not merely the most frequent ones. Equal weighting ensures that each event type contributes uniformly to the final measure, providing a fairer assessment of event-type–level consistency.\\

\begin{table}[t]
\footnotesize
\setlength{\tabcolsep}{0.8mm}
\caption{Component-level inter-annotator agreement (\%).}
\centering
\begin{threeparttable}
\begin{tabular}{lccccc}
 & Presence & Subtype & Start D. & End D. & Mean \\ \toprule
\textbf{P. lg. cancer episode} & 96.77 & 93.55 & 38.71 & 80.64 & 77.42 \\
\textbf{S. cancer episode} & 61.43 & 61.43 & 17.14 & 40.00 & 44.28 \\
\textbf{TKI targeted th.} & 80.85 & 80.85 & 44.68 & 44.68 & 62.76 \\
\textbf{EGFR mutation} & 100.00 & -- & 45.45 & 54.54 & 66.67 \\
\textbf{ALK mutation} & 75.00 & -- & 0.00 & 25.00 & 33.33 \\
\textbf{Lg. cancer disease} & 100.00 & -- & 46.67 & 90.00 & 78.89 \\
\bottomrule
\end{tabular}
        \begin{tablenotes}
    \footnotesize
    \item Each row corresponds to an event category, and the columns report the agreement percentages for each individual component: \textbf{presence}, \textbf{subtype} (when applicable), \textbf{start date}, \textbf{end date}, and the overall \textbf{mean agreement}.
    \item Dashes (--) indicate that the component was not applicable to that event type.
  \end{tablenotes}
    \end{threeparttable}
\label{evaluation_agreement}
\end{table}

The highest overall agreement was observed for \textit{lung cancer} (78.89\%), closely followed by \textit{primary lung cancer episodes} (77.42\%). In both cases, annotators showed near-perfect agreement on the \textbf{presence} of the event (100.00\% and 96.77\%, respectively) and strong agreement on the \textbf{end date} (90.00\% and 80.64\%). This consistency was expected, as these two events are closely related: the presence of a primary tumor implies the presence of the disease. Conversely, the disease is generally considered to begin with the initial primary episode.  
Yet, the two events differ in temporal granularity—multiple primary episodes may occur within the course of a single lung cancer disease—which likely explains the slightly lower agreement on episode end dates compared with the disease as a whole.  
By contrast, agreement on the \textbf{start date} remained much lower (46.67\% for lung cancer disease and 38.71\% for primary episodes), highlighting the difficulty of precisely identifying the onset of events from clinical records, which are often incomplete or implicit regarding temporal boundaries.

In comparison, \textit{secondary cancer episodes} showed significantly lower agreement, with an average score of only 44.28\%. While annotators agreed on \textbf{presence} and \textbf{subtype} in just over 60\% of cases, agreement on \textbf{start date} and \textbf{end date} dropped to 17.14\% and 40.00\%, respectively. This variability likely reflects the greater complexity and ambiguity of metastatic events, which are often inconsistently documented or inferred retrospectively from clinical narratives.

\textit{TKI-targeted therapies} received moderate agreement (62.76\%), with relatively balanced scores across components. Agreement on \textbf{presence} and \textbf{subtype} was relatively high (80.85\% for both), but, once again, temporal boundaries again introduced inconsistencies (44.68\% for both \textbf{start} and \textbf{end}). These discrepancies may stem from variations in documentation practices (e.g. prescription vs. administration dates) or from unclear indications regarding treatment transitions or discontinuations.

For \textit{molecular markers}, the results were mixed. Annotations related to \textit{EGFR mutations} achieved a relatively high average agreement (66.67\%),  driven by perfect agreement on \textbf{presence} (100.00\%) but only moderate alignment on \textbf{start} and \textbf{end dates} (45.45\% and 54.54\%). This likely reflects the unambiguous nature of molecular statuses when explicitly mentioned, while their exact temporal relevance (e.g. test date vs. clinical significance date) may vary. Conversely, \textit{ALK mutations} had the lowest agreement (33.33\%), with 75.00\% agreement on \textbf{presence} but complete disagreement on the \textbf{start date} (0.00\%) and low agreement on the \textbf{end date} (25.00\%). This discrepancy may stem from the relative rarity of ALK mutations in the dataset, resulting in fewer annotated instances and greater variability in interpretation.

Overall, the data in Table~\ref{evaluation_agreement} reveals several key trends:

\begin{itemize}
    \item \textbf{Presence} is the most consistently agreed-upon component across all event types, indicating that clinicians generally agree on whether a clinical event occurred.
    \item \textbf{Subtype} (when applicable) shows moderate agreement but remains subject to interpretation, especially in cases involving complex therapies or metastatic events.
    \item \textbf{Start dates} exhibit the lowest levels of agreement across nearly all event types, confirming that identifying the precise onset of an event is inherently ambiguous and often under-documented in EHRs.
    \item \textbf{End dates} are slightly more reliable, though still variable—possibly because an event’s conclusion is more often linked to clear documentation such as treatment discontinuation or discharge summaries.
\end{itemize}

These results highlight the inherent difficulty of establishing a single “gold standard” for temporally extended clinical events, even among human experts.

\subsubsection{Evaluation of agreement between \casper inferences and expert annotations.}
\label{sec:casper_expert_comparison}

To further evaluate \ncasper’s qualitative alignment with clinical reasoning, we compared the events inferred in its \emph{consistent timelines} against the individual annotations produced by the four experts involved in the lung cancer use case. Each annotator evaluated a subset of patient records, forming two sub-cohorts of 15 patients each (Annotators 1 \& 2, Annotators 3 \& 4). For each pair, \ncasper’s output was compared separately against each annotator’s annotations, and we computed a component-wise agreement (presence, event type, start, end) using the same weighted scheme described earlier. 

It is important to note that the annotators had access to the complete EHR, including narrative reports, whereas \casper only relied on structured EHR data (diagnoses, mutations, TKI therapy administrations or prescriptions, and death). Consequently, some discrepancies may reflect information sources unavailable to \casper rather than reasoning errors.\\

\noindent \textbf{Sub-cohort 1 (Annotators 1 \& 2).}
\casper achieved a mean agreement ratio of \textbf{57.90\%} with Annotator~1 (Table \ref{tab:casper_annot1}), and \textbf{49.81\%} with Annotator~2 (Table \ref{tab:casper_annot2}). The corresponding inter-annotator agreement for this
sub-cohort is \textbf{56.16\%} (Table~\ref{tab:cohort1_annot}).
Thus, \ncasper’s agreement with individual experts falls
within the range of variability observed between the
annotators themselves.

\begin{table}[h!]
\footnotesize
\setlength{\tabcolsep}{0.8mm}
\caption{Comparison between \casper and Annotator 1.}
\label{tab:casper_annot1}
\centering
\begin{tabular}{lccccc}
 & Presence & Subtype & Start D. & End D. & Mean \\ \toprule
\textbf{P. lg cancer episode} & 85.71 & 85.71 & 35.71 & 71.43 & 69.64 \\
\textbf{S. cancer episode} & 59.37 & 59.37 & 00.00 & 34.37 & 38.28 \\
\textbf{TKI targeted th.} & 57.89 & 57.89 & 10.53 & 31.58 & 39.47 \\
\textbf{EGFR mutation} & 88.89 & -- & 33.33 & 77.78 & 66.67 \\
\textbf{ALK mutation} & 66.67 & -- & 33.33 & 66.67 & 55.56 \\
\textbf{Lg. cancer disease} & 100.00 & -- & 41.67 & 91.67 & 77.8 \\ \bottomrule
\end{tabular}
\end{table}

\begin{table}[h!]
\footnotesize
\setlength{\tabcolsep}{0.8mm}
\caption{Comparison between \casper and Annotator 2.}
\label{tab:casper_annot2}
\centering
\begin{tabular}{lccccc}
 & Presence & Subtype & Start D. & End D. & Mean \\ \toprule
\textbf{P. lg. cancer episode} & 85.71 & 85.71 & 35.71 & 71.43 & 69.64 \\
\textbf{S. cancer episode} & 56.25 & 56.25 & 12.50 & 50.00 & 43.75 \\
\textbf{TKI targeted th.} & 60.00 & 60.00 & 25.00 & 45.00 & 47.50 \\
\textbf{EGFR mutation} & 88.89 & -- & 66.67 & 33.33 & 62.96 \\
\textbf{ALK mutation} & 00.00 & -- & 00.00 & 00.00 & 00.00 \\
\textbf{Lg. cancer disease} & 100.00 & -- & 41.67 & 83.33 & 75.00 \\
\bottomrule
\end{tabular}
\end{table}

\begin{table}[h!]
\footnotesize
\setlength{\tabcolsep}{0.8mm}
\caption{Component-level inter-annotator agreement (\%) of sub-cohort 1.}
\label{tab:cohort1_annot}
\centering
\begin{tabular}{lccccc}
 & Presence & Subtype & Start D. & End D. & Mean \\ \toprule
\textbf{P. lg cancer episode} & 100.00 & 100.00 & 40.00 & 93.33 & 83.33 \\
\textbf{S. cancer episode} & 71.87 & 71.87 & 00.00 & 25.00 & 40.62 \\
\textbf{TKI targeted th.} & 86.96 & 86.96 & 52.17 & 30.43 & 64.13 \\
\textbf{EGFR mutation} & 100.00 & -- & 33.33 & 16.67 & 50.00 \\
\textbf{ALK mutation} & 50.00 & -- & 00.00 & 00.00 & 16.67 \\
\textbf{Lg. cancer disease} & 100.00 & -- & 53.33 & 93.33 & 82.22 \\ \bottomrule
\end{tabular}
\end{table}

Annotator 1 reported 11 events missing from \ncasper’s output, while \casper inferred 14 events not identified by the annotator. Annotator 2 reported 10 events not found by \casper, which in turn detected 17 events not annotated by the expert. Overall, \casper\ tended to identify slightly more events than the annotators, which may influence the agreement score since additional detections—although sometimes clinically plausible—were automatically counted as disagreements when they were absent from the expert annotations.

The highest concordance was observed for \emph{primary lung cancer episodes} (69.64\%) and \emph{lung cancer disease} ($\approx$76.40\%) confirming \ncasper’s ability to capture the main clinical trajectory.
By contrast, \emph{secondary cancer episodes} and \emph{TKI therapies} showed lower agreement ($\approx$38–47\%), likely due to missing or delayed documentation of metastases and heterogeneous recording of drug transitions.  
EGFR and ALK mutations exhibited intermediate performance, with better scores for EGFR ($\approx$63–67\%) than for ALK ($\approx$0–56\%), reflecting the rarity and inconsistent testing of ALK in the dataset.\\

\noindent \textbf{Sub-cohort 2 (Annotators 3 \& 4).}
Annotator 3 reached an overall agreement of \textbf{49.69\%} and Annotator 4 of \textbf{51.12\%} (Tables \ref{tab:casper_annot3}–\ref{tab:casper_annot4}). The inter-annotator agreement for this sub-cohort
is higher (\textbf{66.36\%}, Table~\ref{tab:cohort2_annot}),
indicating stronger concordance between the two experts
than in Sub-cohort~1.

\begin{table}[h!]
\footnotesize
\setlength{\tabcolsep}{0.8mm}
\caption{Comparison between \casper and Annotator 3.}
\label{tab:casper_annot3}
\centering
\begin{tabular}{lccccc}
 & Presence & Subtype & Start D. & End D. & Mean \\ \toprule
\textbf{P. lg. cancer episode} & 100.00 & 93.75 & 12.50 & 56.25 & 65.62 \\
\textbf{S. cancer episode} & 42.55 & 42.55 & 8.51 & 36.17 & 32.45 \\
\textbf{TKI targeted th.} & 65.22 & 65.22 & 17.39 & 39.13 & 46.74 \\
\textbf{EGFR mutation} & 90.00 & -- & 60.00 & 60.00 & 70.00 \\
\textbf{ALK mutation} & 50.00 & -- & 00.00 & 00.00 & 16.67 \\
\textbf{Lg. cancer disease} & 100.00 & -- & 20.00 & 80.00 & 66.67 \\ \bottomrule
\end{tabular}
\end{table}

\begin{table}[h!]
\footnotesize
\setlength{\tabcolsep}{0.8mm}
\caption{Comparison between \casper and Annotator 4.}
\label{tab:casper_annot4}
\centering
\begin{tabular}{lccccc}
 & Presence & Subtype & Start D. & End D. & Mean \\ \toprule
\textbf{P. lg. cancer episode} & 93.75 & 93.75 & 25.00 & 56.25 & 67.19 \\
\textbf{S. cancer episode} & 46.30 & 46.30 & 5.56 & 40.74 & 34.72 \\
\textbf{TKI targeted th.} & 66.67 & 66.67 & 12.50 & 33.33 & 44.79 \\
\textbf{EGFR mutation} & 90.00 & -- & 60.00 & 60.00 & 70.00 \\
\textbf{ALK mutation} & 50.00 & -- & 00.00 & 00.00 & 16.67 \\
\textbf{Lg. cancer disease} & 100.00 & -- & 33.33 & 86.67 & 73.33 \\
\bottomrule
\end{tabular}
\end{table}

\begin{table}[h!]
\footnotesize
\setlength{\tabcolsep}{0.8mm}
\caption{Component-level inter-annotator agreement (\%) of sub-cohort 2.}
\label{tab:cohort2_annot}
\centering
\begin{tabular}{lccccc}
 & Presence & Subtype & Start D. & End D. & Mean \\ \toprule
\textbf{P. lg cancer episode} & 93.75 & 87.50 & 37.50 & 68.75 & 71.87 \\
\textbf{S. cancer episode} & 52.63 & 52.63 & 00.00 & 31.58 & 52.63 \\
\textbf{TKI targeted th.} & 75.00 & 75.00 & 37.50 & 58.33 & 61.46 \\
\textbf{EGFR mutation} & 100.00 & -- & 60.00 & 100.00 & 86.67 \\
\textbf{ALK mutation} & 100.00 & -- & 00.00 & 50.00 & 50.00 \\
\textbf{Lg. cancer disease} & 100.00 & -- & 40.00 & 86.67 & 75.56 \\ \bottomrule
\end{tabular}
\end{table}

\casper\ again identified more events (25 and 22, respectively) than the annotators (10 and 18).  
The best-aligned events remained the \emph{primary lung cancer episodes} and \emph{lung cancer disease} ($\approx$65–73\%), both directly supported by structured diagnostic codes.  
Lower agreement for \emph{secondary cancer episodes} and \emph{TKI therapies} reflects the greater uncertainty in documenting metastatic progression and treatment continuity—both of which are often inferred retrospectively from scattered clinical entries rather than explicitly timestamped observations.  
Mutations showed moderate alignment for EGFR and very limited for ALK, a difference mainly attributable to data sparsity and heterogeneity in genetic testing practices: EGFR testing was systematically performed and recorded, whereas ALK testing was less frequent and typically not conducted once an EGFR mutation had been identified.\\

\noindent \textbf{Merged pair evaluation.}
When merging both annotators’ results within each sub-cohort by keeping, for each patient and event, the best agreement score, the overall agreement between \casper and human expert annotation reached \textbf{60.25\%}.  
This value approaches the inter-annotator agreement (60.56\%), indicating that \ncasper’s inferences are roughly as consistent with expert judgments as the experts are among themselves—an encouraging result supporting the clinical validity of its outputs.\\

Across all comparisons, \casper reproduced expert-level reasoning for well-structured events directly linked to coded data (\emph{primary lung cancer episodes}, \emph{lung cancer disease}, \emph{EGFR mutation}).  
Lower scores for \emph{secondary cancer episodes} and \emph{TKI therapies} highlight limitations related to incomplete data coverage.  
Disagreements on temporal boundaries (start/end) also mirror the annotators’ own divergence, confirming that imprecision largely stems from the inherent ambiguity of clinical documentation and from the intrinsic difficulty of this task, which remains challenging even for human experts.  
Overall, these results demonstrate that \casper can emulate the temporal reasoning patterns of clinicians using only structured data and logically specified rules, producing outputs of comparable reliability to human annotations.\\

\noindent
\textbf{Justification of the evaluation method.}
Our component-level agreement approach is consistent with methodologies adopted in previous studies addressing complex biomedical or temporal annotation tasks~\cite{hripcsak_Agreement_2005,artstein_InterCoder_2008,uzuner_Evaluating_2007}. In such contexts, standard agreement metrics like Cohen’s kappa often prove inadequate, as they assume flat categorical labels and fail to account for the internal structure or temporal boundaries of clinical events. Prior work has advocated for alternative strategies better suited to structured annotations, including element-level or boundary-aware scoring schemes~\cite{hripcsak_Agreement_2005,artstein_InterCoder_2008,uzuner_Evaluating_2007}. These precedents support the relevance of our method for capturing meaningful agreement in richly structured annotation scenarios.


\end{document}